\documentclass[sigconf]{acmart}
\AtBeginDocument{%
  }

\copyrightyear{2025}
\acmYear{2025}
\setcopyright{cc}
\setcctype{by}
\acmConference[CIKM '25]{Proceedings of the 34th ACM International Conference on Information and Knowledge Management}{November 10--14, 2025}{Seoul, Republic of Korea}
\acmBooktitle{Proceedings of the 34th ACM International Conference on Information and Knowledge Management (CIKM '25), November 10--14, 2025, Seoul, Republic of Korea}\acmDOI{10.1145/3746252.3761094}
\acmISBN{979-8-4007-2040-6/2025/11}
\usepackage{enumitem}
\usepackage{adjustbox}
\usepackage{colortbl}
\usepackage{array,multirow,graphicx}
\usepackage{algorithmic}
\usepackage[ruled,vlined]{algorithm2e}
\usepackage{caption, subcaption}
\usepackage{flushend}

\DeclareCaptionFont{subcapfont}{\fontsize{12pt}{12pt}\selectfont}

\allowdisplaybreaks[4]

\newtheorem{dfn}{\textbf{Definition}}

\newtheorem{prop}{\textbf{Proposition}}

\newcommand{\sC}[0]{\mathcal{C}}
\newcommand{\sG}[0]{\mathcal{G}}
\newcommand{\sV}[0]{\mathcal{V}}
\newcommand{\sE}[0]{\mathcal{E}}
\newcommand{\sEv}[0]{\mathcal{E}_v}
\newcommand{\va}[0]{\boldsymbol{a}}

\newcommand{\vh}[0]{\boldsymbol{h}}
\newcommand{\vm}[0]{\boldsymbol{m}}
\newcommand{\vx}[0]{\boldsymbol{x}}
\newcommand{\vz}[0]{\boldsymbol{z}}

\newcommand{\mH}[0]{\boldsymbol{H}}
\newcommand{\mI}[0]{\boldsymbol{I}}

\newcommand{\mM}[0]{\boldsymbol{M}}

\newcommand{\mP}[0]{\boldsymbol{P}}

\newcommand{\mR}[0]{\mathcal{R}}

\newcommand{\mX}[0]{\boldsymbol{X}}
\newcommand{\mZ}[0]{\boldsymbol{Z}}

\DeclareMathOperator*{\argmin}{arg\,min}

\newcommand{\bigO}[0]{\mathcal{O}}
\newcommand{\fVE}[0]{f_{\sV \rightarrow \sE}}
\newcommand{\fEV}[0]{f_{\sE \rightarrow \sV}}
\def\eg{\emph{e.g.}}
\def\ie{\emph{i.e.}}

\usepackage{pifont}
\newcommand{\cmark}{\textcolor[rgb]{0,0.6,0}{\ding{51}}}%
\newcommand{\xmark}{\textcolor[rgb]{1,0,0}{\ding{55}}}%

\usepackage{xcolor}   
\definecolor{lightgray}{gray}{0.9}
\definecolor{darkblue}{RGB}{0,0,120}

\usepackage{soul}
\setstcolor{blue}
\soulregister{\cite}7
\soulregister{\citep}7
\soulregister{\citet}7
\soulregister{\autoref}7
\soulregister{\pageref}7

\begin{document}

% \title{Co-Representation Neural Hypergraph Diffusion for Edge-Dependent Node Classification}
\title{Modeling Edge-Specific Node Features through Co-Representation Neural Hypergraph Diffusion}

\author{Yijia Zheng}
\orcid{0000-0002-6585-8273}
\affiliation{%
  \institution{University of Amsterdam}
  \city{Amsterdam}
  % \state{Noord-Holland}
  \country{The Netherlands}
}
\email{y.zheng@uva.nl}

\author{Marcel Worring}
\orcid{0000-0003-4097-4136}
\affiliation{%
  \institution{University of Amsterdam}
  \city{Amsterdam}
  \country{The Netherlands}
}
\email{m.worring@uva.nl}

\renewcommand{\shortauthors}{Yijia Zheng and Marcel Worring}

\begin{abstract}
  Hypergraphs are widely being employed to represent complex higher-order relations in real-world applications. Most existing research on hypergraph learning focuses on node-level or edge-level tasks. A practically relevant and more challenging task, edge-dependent node classification (ENC), is still under-explored. In ENC, a node can have different labels across different hyperedges, which requires the modeling of node features unique to each hyperedge. The state-of-the-art ENC solution, WHATsNet, only outputs single node and edge representations, leading to the limitations of \textbf{entangled edge-specific features} and \textbf{non-adaptive representation sizes} when applied to ENC. Additionally, WHATsNet suffers from the common \textbf{oversmoothing issue} in most HGNNs. To address these limitations, we propose \textbf{CoNHD}, a novel HGNN architecture specifically designed to model edge-specific features for ENC. Instead of learning separate representations for nodes and edges, CoNHD reformulates within-edge and within-node interactions as a hypergraph diffusion process over node-edge co-representations. We develop a neural implementation of the proposed diffusion process, leveraging equivariant networks as diffusion operators to effectively learn the diffusion dynamics from data. Extensive experiments demonstrate that CoNHD achieves the best performance across all benchmark ENC datasets and several downstream tasks without sacrificing efficiency. Our implementation is available at \url{https://github.com/zhengyijia/CoNHD}.
\end{abstract}





\maketitle

%%%%%%%%%%%%%%%%%%%%%%%%%%%%%%%%%%%%%%%%%%%%%%%%%%%%%%%%%%%%%%%%%%%%%%%%%%%%%%%
% Introduction
%%%%%%%%%%%%%%%%%%%%%%%%%%%%%%%%%%%%%%%%%%%%%%%%%%%%%%%%%%%%%%%%%%%%%%%%%%%%%%%

\section{Introduction}\label{sec:introduction}

Real-world applications often involve intricate higher-order relations that cannot be represented by traditional graphs with pairwise connections \citep{battiston2020networks, yang2023group, fanseu2021hypergraph}. Hypergraphs, where an edge can connect more than two nodes, provide a flexible structure for representing such relations \citep{gao2020hypergraph, antelmi2023survey}. To tackle hypergraph-related tasks such as node classification \cite{liu2024hypergraph, zou2024unig, benko2024hypermagnet} and edge prediction \cite{kim2024hypeboy, jo2021edge, sun2021multi, chen2022explainable}, message passing-based hypergraph neural networks (HGNNs) have become the common solution \citep{kim2024survey}. Recent research \citep{huang2021unignn, chien2022you} shows that most HGNNs can be formulated as an instantiation of the two-stage message passing framework depicted in Fig.~\ref{fig:hgnn_architectures}(a). The first stage aggregates messages from nodes to update the edge representation, while the second stage aggregates messages from edges to update the node representation. Although message passing-based HGNNs have achieved success in various applications \citep{kim2024survey, chen2024hyperedge, saifuddin2023seq}, the majority of research efforts have concentrated on node-level and edge-level tasks. In many real-world hypergraphs, a node's property varies with different hyperedges it belongs to. For instance, in a co-authorship network, a researcher may be the lead author in one paper but the corresponding author in another. Likewise, in a multiplayer game, a player might be the winner in one match yet the loser in another. Motivated by such scenarios, \citet{choe2023classification} introduce a new task namely \textit{edge-dependent node classification} (ENC), where a node can have different labels across different hyperedges. This new task has been shown to be valuable for many downstream tasks \citep{choe2023classification}, including ranking aggregation \citep{chitra2019random}, node clustering \citep{hayashi2020hypergraph}, and product-return prediction \citep{li2018tail}. 

Although many message passing-based HGNNs can be applied to ENC, \citet{choe2023classification} highlight that these methods overlook edge-specific node features during aggregation. To address this limitation, they propose WHATsNet, the current state-of-the-art method for ENC. WHATsNet follows the edge-dependent message passing framework as shown in Fig.~\ref{fig:hgnn_architectures}(b), where edge-dependent representations are extracted before aggregation. The final node and edge representations are concatenated to predict the ENC labels. Adopting the dominant message passing framework to address ENC is intuitive, but does it yield the most appropriate solution?

\begin{figure*}[t]
\begin{center}
\includegraphics[trim={0cm 0cm 0cm 0cm},clip,width=0.9\linewidth]{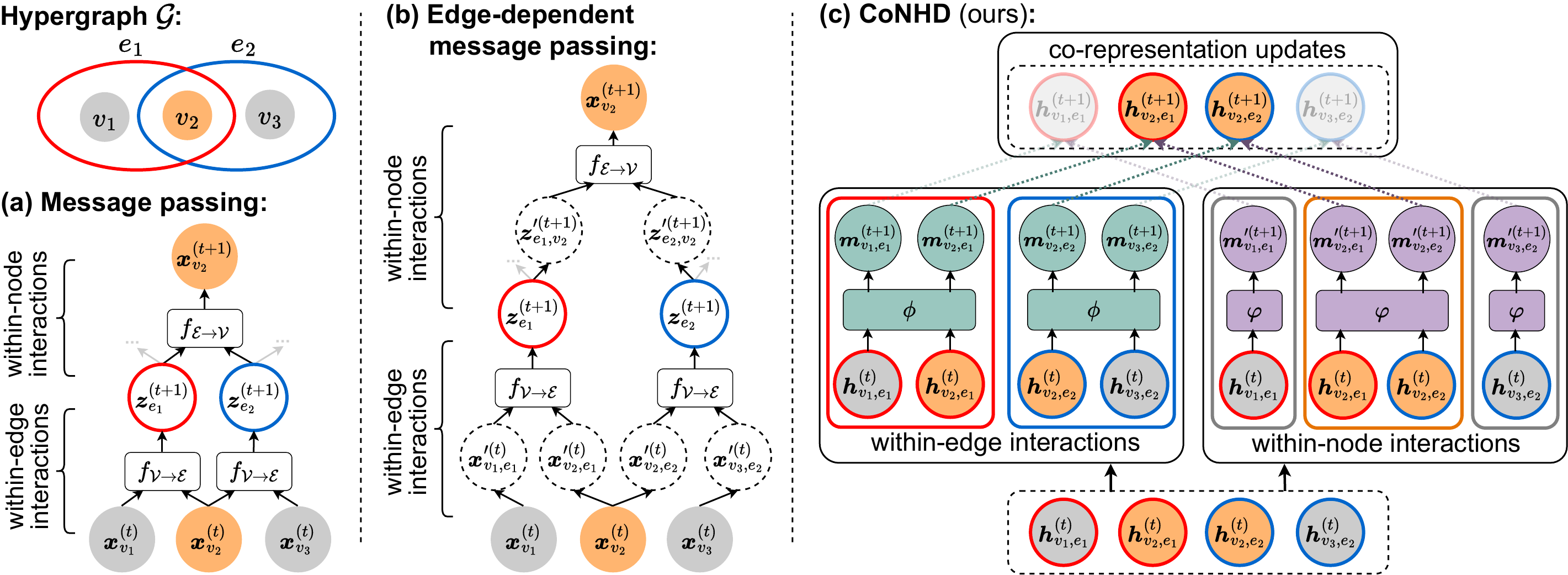}
\end{center}
% \vspace{-3mm}
\caption{\textbf{Different HGNN architectures.} (a,b) The (edge-dependent) message passing framework aggregates (edge-dependent) messages from neighboring nodes to update a single edge representation through an aggregation function $\fVE$ and then from neighboring edges back to update a single node representation through an aggregation function $\fEV$. (c) Our proposed CoNHD redefines within-edge and within-node interactions as multi-input multi-output processes among node-edge co-representations. These interactions are modeled by two equivariant networks, $\phi$ and $\varphi$, which can generate diverse node-specific or edge-specific information for different node-edge pairs. The outputs from both interactions are used to update the co-representations.}
\label{fig:hgnn_architectures}
% \vspace{-3mm}
\end{figure*}

Unlike traditional node-level or edge-level tasks, ENC allows a node to have varying labels across different hyperedges. This requires, as indicated in \citep{choe2023classification}, the model to capture node features unique to each hyperedge. However, the message passing framework aggregates different edge-specific features into a single node representation, leading to the following two limitations:

\textbf{(1) Entangled edge-specific features.} The single node representation entangles edge-specific features from different edges, making it challenging to distinguish features corresponding to a specific target edge. This becomes particularly problematic when the edge-specific features are highly dissimilar, as the entangled vector may obscure features specific to different edges. To verify this assumption, as shown in Figure~\ref{fig:emailenron_entropy_microf1}, we examine the performance of WHATsNet under different node entropy levels, where higher entropy levels indicate that the node has more dissimilar labels in different edges. At low entropy levels, since a node has similar labels in neighboring edges, the prediction may rely on similar features, and therefore WHATsNet with single node representations performs well. As the entropy level increases, dissimilar edge-specific features are required to predict different labels. The performance of WHATsNet drops significantly, which supports our assumption that a single node representation with entangled edge-specific features is insufficient for predicting different ENC labels.

\textbf{(2) Non-adaptive representation sizes.} Storing different edge-specific features in a fixed-size node representation vector causes information loss for large-degree nodes, which interact with more neighboring hyperedges and therefore require larger representation sizes. As shown in Figure~\ref{fig:visualization_embeddings}, WHATsNet fails to generate discriminative embeddings for node-edge pairs related to large-degree nodes. Since low-degree nodes have fewer neighboring edges and do not require large representation sizes, simply increasing the embedding dimension for all nodes leads to excessive computational costs and problems like overfitting and optimization difficulties \citep{luo2021graph, goodfellow2016deep}.

Apart from the above two limitations specific to the ENC task, WHATsNet \citep{choe2023classification} also suffers from the common \textbf{oversmoothing issue} in most HGNNs \citep{wang2023equivariant, yan2024hypergraph}, as demonstrated in Fig.~\ref{fig:depth_experiment}. This issue hinders the utilization of long-range information and limits model performance. Unlike traditional HGNNs, hypergraph diffusion methods \citep{liu2021strongly, fountoulakis2021local, veldt2023augmented} obtain optimal node representations by directly optimizing a regularized objective function, ensuring convergence to the desired solution. \citet{wang2023equivariant} propose an HGNN inspired by hypergraph diffusion, demonstrating its robustness to the oversmoothing issue. However, their approach remains within the message passing framework and inherits the two aforementioned limitations when applied to ENC.

To overcome the limitations of message passing for ENC, we introduce \underline{\textbf{Co}}-representation \underline{\textbf{N}}eural \underline{\textbf{H}}ypergraph \underline{\textbf{D}}iffusion (\textbf{CoNHD}), a novel diffusion-based HGNN architecture for modeling edge-specific features. Specifically, we show that the two aforementioned limitations are both related to the single-output design in message passing as shown in Fig.~\ref{fig:single_output_multi_output}(a), which only generates a single node or edge representation. Therefore, we first extend the concept of hypergraph diffusion by utilizing node-edge co-representations, redefining the input and output of within-edge and within-node interactions as information exchanged across multiple node-edge pairs, as shown in Fig.~\ref{fig:single_output_multi_output}(b). The co-representation design enables each node to have multiple representations, and the number of these representations scales with the node degree. We further develop a neural implementation that leverages learnable equivariant networks as diffusion operators, which can adaptively learn suitable diffusion dynamics and effectively capture diverse edge-specific features, eliminating the need for handcrafting regularization functions.
\textbf{Our main contributions} are twofold: 
% \vspace{-1mm}
\begin{enumerate}[leftmargin=*, topsep=1mm]
\item We define \textbf{co-representation hypergraph diffusion}, a new concept that generalizes hypergraph diffusion using node-edge co-representations, which offers the benefits of disentangled edge-specific features and adaptive representation sizes. 
\item We propose \textbf{CoNHD}, a neural implementation of the proposed diffusion process. This results in a novel HGNN architecture that can learn diffusion dynamics from data and effectively capture edge-specific features for addressing the ENC task.
\end{enumerate}
We conduct extensive experiments to validate the effectiveness and efficiency of CoNHD, demonstrating that CoNHD achieves the best performance across ten ENC datasets as well as several downstream tasks while maintaining high efficiency.

\begin{figure}[t]
\centering
% \vspace{-12pt}
\includegraphics[width=0.85\linewidth]{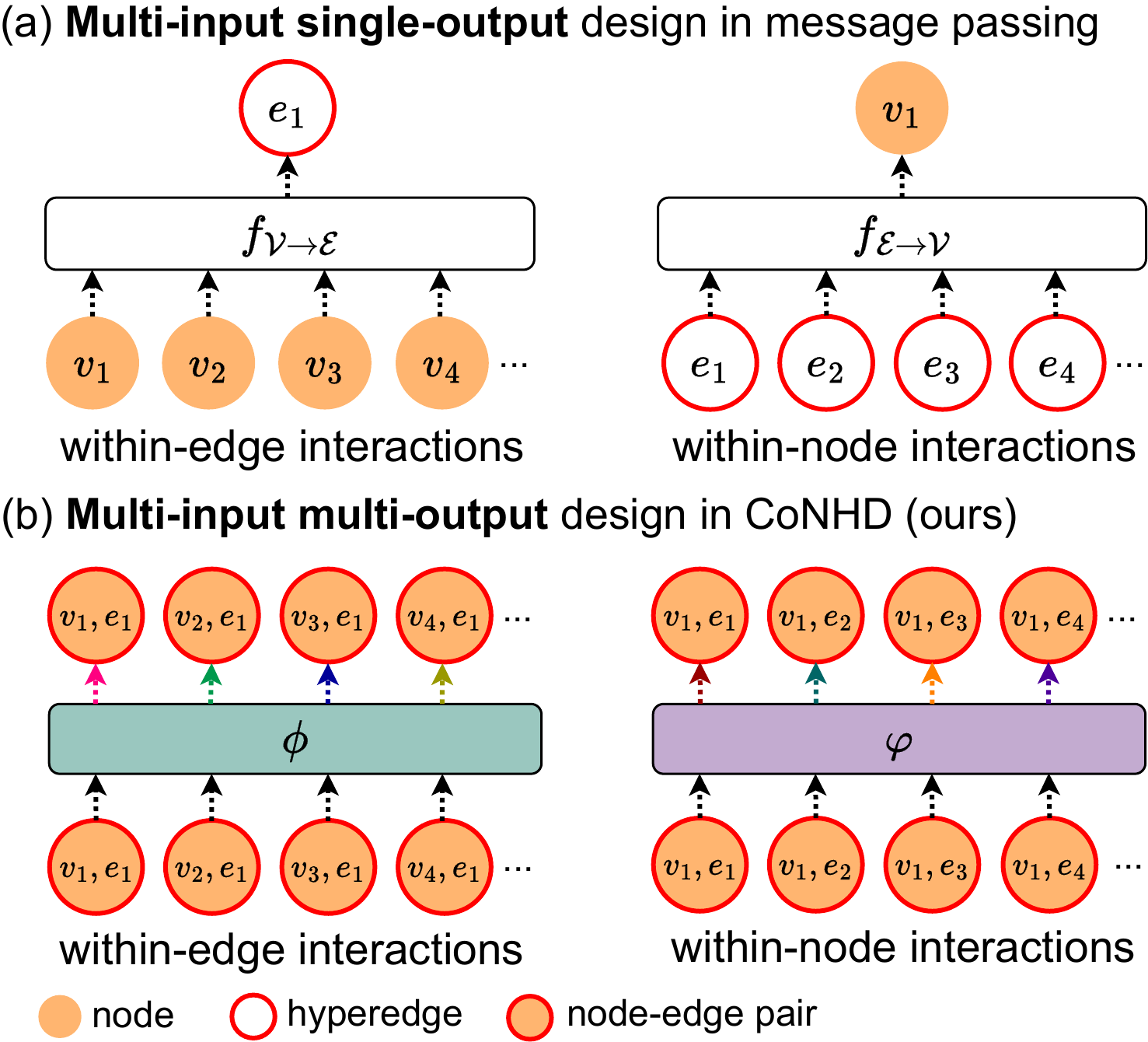}
% \vspace{-3mm}
\caption{\textbf{Comparison between the single-output design in message-passing and the multi-output design in our CoNHD method. (a) In the single-output design, information from multiple neighboring nodes or edges is aggregated into a single edge or node using the aggregation function $\fVE$ or $\fEV$, respectively. (b) In our design, information diffuses across node-edge pairs using two multi-input multi-output functions $\phi$ and $\varphi$. These functions are designed as equivariant, which can produce diverse outputs while maintaining element-wise consistency under permutation.}}
% \vspace{-5mm}
\label{fig:single_output_multi_output}
\end{figure}

%%%%%%%%%%%%%%%%%%%%%%%%%%%%%%%%%%%%%%%%%%%%%%%%%%%%%%%%%%%%%%%%%%%%%%%%%%%%%%%
% Related Work
%%%%%%%%%%%%%%%%%%%%%%%%%%%%%%%%%%%%%%%%%%%%%%%%%%%%%%%%%%%%%%%%%%%%%%%%%%%%%%%
% \vspace{-1mm}
\section{Related Work}\label{sec:related_work}
% \vspace{-1mm}

\textbf{Hypergraph Neural Networks.} Inspired by the success of graph neural networks (GNNs) \citep{kipf2017semisupervised, wu2020comprehensive, wu2022graph}, hypergraph neural networks (HGNNs) have been proposed for modeling complex higher-order relations \citep{kim2024survey, duta2023sheaf}. HyperGNN \citep{feng2019hypergraph, gao2022hgnn+} and HCHA \citep{bai2021hypergraph} define hypergraph convolution based on the clique expansion graph. HyperGCN \citep{yadati2019hypergcn} reduces the clique expansion graph into an incomplete graph with mediators. To directly utilize higher-order structures, HNHN \citep{dong2020hnhn} and HyperSAGE \citep{arya2020hypersage, arya2024adaptive} model the convolution layer as a message passing process with two aggregation stages. UniGNN \citep{huang2021unignn} provides a general framework for extending GNNs to hypergraphs. AllSet~\citep{chien2022you} implements the aggregation functions in message passing as universal invariant functions. HDS$^{ode}$ improves message passing by modeling it as an ODE-based dynamic system \citep{yan2024hypergraph}. Recent research explores edge-dependent message passing, where edge-dependent node messages are extracted before feeding them into the aggregation process \citep{aponte2022hypergraph, wang2023equivariant, telyatnikov2023hypergraph}. LEGCN \citep{yang2022semi} and MultiSetMixer \citep{telyatnikov2023hypergraph} have multiple representations for a single node. However, these two methods model interactions as an invariant function, which only produces the same propagating messages for different node-edge pairs. This invariance design, as shown in Section~\ref{sec:exp_ablation_study}, is insufficient to capture the edge-specific features for ENC. While most existing methods focus on node-level or edge-level tasks \citep{liu2024hypergraph, benko2024hypermagnet, chen2023teasing, behrouz2023cat} and applications \citep{yang2024behavior, wei2022dynamic, shu2024llm, cao2024hypergraph}, the ENC task remains less explored. \citet{choe2023classification} explore the ENC task and propose WHATsNet, the state-of-the-art solution based on message passing. Different from our method, it employs an aggregation after the equivariant operator to produce a single node or edge representation. While message passing has become a dominant framework for addressing various hypergraph-related tasks \cite{kim2024survey}, its single node and edge representation design suffers from the limitations of entangled edge-specific features and non-adaptive representation sizes when applied to ENC. 

\textbf{(Hyper)graph Diffusion.} Different from HGNNs with trainable parameters, hypergraph diffusion is a class of non-parametric regularization methods. (Hyper)graph diffusion \citep{gleich2015using, chamberlain2021grand} models the diffusion information as the gradients derived from minimizing a regularized target function, which regularizes the node representations within the same edge. This ensures that the learned node representations converge to the solution of the optimization target instead of an oversmoothed solution \citep{yang2021graph, thorpe2022grand++}. The technique was first introduced to achieve local and global consistency on graphs \citep{zhu2003semi, zhou2003learning}, and was then generalized to hypergraphs \citep{zhou2007learning, antelmi2023survey}. \citet{zhou2007learning} propose a regularization function by reducing the higher-order structure in a hypergraph using clique expansion. To directly utilize the higher-order structures, \citet{hein2013total} propose a regularization function based on the total variation of the hypergraph. Other regularization functions are designed to improve parallelization ability and introduce non-linearity \citep{jegelka2013reflection, tudisco2021anonlinear, tudisco2021nonlinear, liu2021strongly}. Some advanced optimization techniques have been investigated in hypergraph diffusion to improve efficiency \citep{zhang2017re, li2020quadratic}. Recent research \citep{chamberlain2021grand, li2022simple, thorpe2022grand++, gravinaanti2023anti, wang2023equivariant, wang2023hypergraph} explores the neural implementation of (hyper)graph diffusion processes, which demonstrate strong robustness against the oversmoothing issue. While hypergraph diffusion methods have shown effectiveness in various tasks like ranking \citep{li2017inhomogeneous}, motif clustering \citep{takai2020hypergraph}, and signal processing \citep{zhang2019introducing, schaub2021signal}, they are restricted to node representations and cannot address the ENC task. 

In this paper, we extend hypergraph diffusion using node-edge co-representations and propose a neural implementation. Most related to our work is ED-HNN \citep{wang2023equivariant}, which is designed to approximate any traditional hypergraph diffusion process. However, it still follows message passing with single node and edge representations. In contrast, our method directly models interactions among co-representations using multi-input multi-output equivariant functions, effectively capturing edge-specific features and achieving significant performance improvements on the ENC task. 

%%%%%%%%%%%%%%%%%%%%%%%%%%%%%%%%%%%%%%%%%%%%%%%%%%%%%%%%%%%%%%%%%%%%%%%%%%%%%%%
% Preliminaries
%%%%%%%%%%%%%%%%%%%%%%%%%%%%%%%%%%%%%%%%%%%%%%%%%%%%%%%%%%%%%%%%%%%%%%%%%%%%%%%
\section{Preliminaries}\label{sec:preliminaries}

In this section, we introduce the general notations for hypergraphs and present key concepts related to message passing-based HGNNs and traditional node-representation hypergraph diffusion, which are essential for the development of our method. 

\textbf{Notations.} Let $\sG = (\sV, \sE)$ denote a hypergraph, where $\sV = \{ v_1, v_2, \ldots, v_n \}$ represents a set of $n$ nodes, and $\sE = \{ e_1, e_2, \ldots, e_m \}$ represents a set of $m$ hyperedges. Each edge $e_i \in \sE$ is a non-empty subset of $\sV$ and can contain an arbitrary number of nodes. $\sEv = \{ e \in \sE | v \in e \}$ represents the set of edges that contain node $v$, and $d_v = \lvert \sEv \rvert$ and $d_e = \lvert e \rvert$ are the degrees of node $v$ and edge $e$, respectively. We use $v^{e}_i$ and $e^{v}_j$ to respectively denote the $i$-th node in edge $e$ and the $j$-th edge in $\sEv$. $\mX^{(0)} = [ \vx^{(0)}_{v_1}, \ldots, \vx^{(0)}_{v_n} ]^{\top}$ is the initial node feature matrix. 

Since the nodes and edges in a hypergraph are inherently unordered, it is important to ensure that the outputs of the interaction modeling functions are consistent regardless of the input ordering. This requirement is formally captured by two key properties: permutation invariance and permutation equivariance. A permutation invariant function is suitable for single-output settings, where the final output remains unchanged under input reordering. In contrast, a permutation equivariant function is well-suited for multi-output settings where permuting the inputs induces the same permutation in the multiple outputs with element-wise consistency. Here we formally give the definitions of these two properties. Let $\mathbb{S}_n$ denote the symmetric group on $n$ elements, where each action $\pi \in \mathbb{S}_n$ acts on any input matrix $\mI \in \mathbb{R}^{n \times d}$ by permuting its rows. 

\begin{dfn}[Permutation Invariance]\label{dfn:permutation_invariance}
A function $g: \mathbb{R}^{n \times d} \\ \rightarrow \mathbb{R}^{d'}$ is permutation invariant if it satisfies $g( \pi \cdot \mI ) = g(\mI)$ for all $\pi \in \mathbb{S}_n$ and $\mI \in \mathbb{R}^{n \times d}$.
\end{dfn}

\begin{dfn}[Permutation Equivariance]\label{dfn:permutation_equivariance}
A function $g: \mathbb{R}^{n \times d} \rightarrow \mathbb{R}^{n \times d'}$ is permutation equivariant if it satisfies $g( \pi \cdot \mI ) = \pi \cdot g(\mI)$ for all $\pi \in \mathbb{S}_n$ and $\mI \in \mathbb{R}^{n \times d}$.
\end{dfn}

\textbf{Message Passing-based HGNNs.} Message passing \citep{huang2021unignn, chien2022you} has become a standard framework for most HGNNs, which models the interactions in within-edge and within-node structures as two multi-input single-output aggregation functions $\fVE$ and $\fEV$: 
\begin{align}
\vz^{(t+1)}_{e} =& \fVE \big( \mX^{(t)}_e; \vz^{(t)}_{e} \big), \label{eq:message_passing_1} \\
\vx^{(t+1)}_{v} =& \fEV \big( \mZ^{(t+1)}_v; \vx^{(t)}_{v} \big). \label{eq:message_passing_2}
% \vx^{(t+1)}_{v} =& \fskip \big( \tilde{\vx}^{(t+1)}_{v}, \vx^{(t)}_{v}, \vx^{(0)}_{v} \big). \label{eq:message_passing_3}
\end{align}
Here $\vx^{(t)}_{v}$ and $\vz^{(t)}_{e}$ are the node and edge representations in the $(t)$-th iteration. $\vx^{(0)}_{v}$ is the initial node features, and $\vz^{(0)}_{e}$ is typically initialized to a zero vector. $\mX^{(t)}_e$ denotes the representations of nodes contained in edge $e$, \ie, $\mX^{(t)}_e = \big[ \vx^{(t)}_{v^{e}_{1}}, \ldots, \vx^{(t)}_{v^{e}_{d_e}} \big]^{\top}$. Similarly, $\mZ^{(t)}_v = \big[ \vz^{(t)}_{e^{v}_{1}}, \ldots, \vz^{(t)}_{e^{v}_{d_v}} \big]^{\top}$ denotes the representations of edges containing node $v$. $\fVE$ and $\fEV$ are two invariant functions that take multiple representations from neighboring nodes or edges as inputs, but only output a single edge or node representation. 

\textbf{Hypergraph Diffusion.} Hypergraph diffusion learns node representations $\mX = \big[ \vx_{v_1}, \ldots, \vx_{v_n} \big]^{\top}$, where $\vx_{v_i} \in \mathbb{R}^{d}$, by minimizing a hypergraph-regularized target function \citep{tudisco2021nonlinear, prokopchik2022nonlinear}. For brevity, we use $\mX_{e} = \big[ \vx_{v^{e}_{1}}, \ldots, \vx_{v^{e}_{d_e}} \big]^{\top}$ to denote the representations of nodes contained in edge $e$. The target function is the weighted summation of some non-structural and structural regularization functions. The non-structural regularization function is independent of the hypergraph structure, which is typically defined as a squared loss function between the learned node representation vector $\vx_{v}$ and the node attribute vector $\va_v$ (composed of initial node features $\vx^{(0)}_{v}$ \citep{takai2020hypergraph} or observed node labels \citep{tudisco2021nonlinear}). The structural regularization functions incorporate the hypergraph structure and apply regularization to multiple node representations within the same hyperedge, which are invariant functions. Many structural regularization functions are designed by heuristics \citep{zhou2007learning, hein2013total, hayhoe2023transferable, tudisco2021anonlinear}. For instance, the clique expansion (CE) regularization functions \citep{zhou2007learning}, defined as $\Omega_{\mathrm{CE}}( \mX_e ) := \sum_{v, u \in e} \lVert \vx_v - \vx_u \rVert_2^2$, encourage the representations of all nodes in the same hyperedge to become similar. Alternatively, the total variation (TV) functions, defined as $\Omega_{\mathrm{TV}}( \mX_e ) := \max_{v, u \in e} \lVert \vx_v - \vx_u \rVert^{p} (p \in \{ 1, 2 \})$, focus on reducing the discrepancy between the most dissimilar nodes within an edge. Without making a choice among these functions, here we discuss the general form of hypergraph diffusion, which can be defined as: 

% which are also referred to as potential functions or energy functions \citep{wang2023equivariant, wang2023hypergraph}

\begin{dfn}[Node-Representation Hypergraph Diffusion]\label{dfn:node_representation_hypergraph_diffusion}
Given a non-structural regularization function $\mR_v(\cdot; \va_v): \mathbb{R}^{d} \rightarrow \mathbb{R}$ and a structural regularization function $\Omega_e(\cdot): \mathbb{R}^{d_e \times d} \rightarrow \mathbb{R}$, the node-representation hypergraph diffusion learns representations by solving the following optimization problem 
\begin{equation}\label{eq:node_representation_hypergraph_diffusion}
\mX^{\star} = \argmin_{\mX} \left\{ \sum_{v \in \sV} \mR_v \big( \vx_v; \va_v \big) + \lambda \sum_{e \in \sE} \Omega_e \big( \mX_e \big) \right\}.
\end{equation}
\end{dfn}
%\vspace{-2mm}
Here $\Omega_e(\cdot)$ is also referred to as the edge regularization function. $\mX^{\star}$ denotes the matrix of all learned node representations, which can be used for predicting the node labels.

%%%%%%%%%%%%%%%%%%%%%%%%%%%%%%%%%%%%%%%%%%%%%%%%%%%%%%%%%%%%%%%%%%%%%%%%%%%%%%%
% Methodology
%%%%%%%%%%%%%%%%%%%%%%%%%%%%%%%%%%%%%%%%%%%%%%%%%%%%%%%%%%%%%%%%%%%%%%%%%%%%%%%
% \vspace{-1mm}
\section{Methodology}\label{sec:methodology}
% \vspace{-1mm}

In this section, we propose a new hypergraph diffusion process based on node-edge co-representations, and then develop CoNHD, a learnable neural implementation of the proposed diffusion process. This leads to the novel HGNN architecture shown in Fig.~\ref{fig:hgnn_architectures}(c).

%%%%%%%%%%%%%%%%%%%%%%%%%%%%%%%%%%%%%%%%%
% Co-representation Hypergraph Diffusion
%%%%%%%%%%%%%%%%%%%%%%%%%%%%%%%%%%%%%%%%%
\subsection{Co-Representation Hypergraph Diffusion}\label{sec:co_representation_hypergraph_diffusion}

In this section, we introduce the co-representation hypergraph diffusion process for modeling edge-specific features in ENC. We first formally define the ENC task following \citep{choe2023classification}.

\begin{dfn}[Edge-Dependent Node Classification (ENC)]  \label{dfn:ENC}
Given \textbf{(1)} a hypergraph $\sG = ( \sV, \sE )$, \textbf{(2)} a label space $\sC$, \textbf{(3)}~observed edge-dependent node labels for $\sE' \subset \sE$ (\ie, $y_{v, e} \in \sC$, $\forall v \in e, \forall e \in \sE'$), and \textbf{(4)} an initial node feature matrix $\mX^{(0)}$, the ENC task is to predict the unobserved edge-dependent node labels for $\sE \setminus \sE'$ (\ie, $y_{v, e} \in \sC$, $\forall v \in e, \forall e \in \sE \setminus \sE'$). 
\end{dfn}

In ENC, the label $y_{v, e}$ is associated with both node $v$ and edge $e$. We extend hypergraph diffusion to learn a co-representation $\vh_{v, e} \in \mathbb{R}^{d}$ for each node-edge pair $(v, e)$. Let $\mH = \big[ \ldots, \vh_{v, e}, \ldots \big]^{\top}$ be the matrix containing co-representation vectors of all node-edge pairs. We use $\mH_e = \big[ \vh_{v^{e}_1, e}, \ldots, \vh_{v^{e}_{d_e}, e} \big]^{\top}$ and $\mH_v = \big[ \vh_{v, e^{v}_1}, \ldots, \vh_{v, e^{v}_{d_v}} \big]^{\top}$ to denote the co-representations associated with edge $e$ and node $v$, respectively. With these notations, the co-representation hypergraph diffusion is defined as:

\begin{dfn}[Co-Representation Hypergraph Diffusion]\label{dfn:co_representation_hypergraph_diffusion}
Given a non-structural regularization function $\mR^{co}_{v,e}(\cdot; \va_{v, e}) : \mathbb{R}^{d} \rightarrow \mathbb{R}$, structural regularization functions $\Omega^{co}_{e}(\cdot) : \mathbb{R}^{d_e \times d} \rightarrow \mathbb{R}$ and $\Omega^{co}_{v}(\cdot) : \mathbb{R}^{d_v \times d} \rightarrow \mathbb{R}$, the co-representation hypergraph diffusion learns node-edge co-representations by solving the following optimization problem 
\begin{equation}\label{eq:co_representation_hypergraph_diffusion}
\begin{split}
\mH^{\star} =& \argmin_{\mH} \Bigg\{ \sum_{v \in \sV} \sum_{e \in \sEv} \mR^{\mathrm{co}}_{v,e} \big( \vh_{v, e}; \va_{v, e} \big)\\ 
& + \lambda \sum_{e \in \sE} \Omega^{\mathrm{co}}_{e} \big( \mH_e \big) + \gamma \sum_{v \in \sV} \Omega^{\mathrm{co}}_{v} \big( \mH_v \big) \Bigg\}. 
\end{split}
% \nonumber
\end{equation}
\end{dfn}
% \vspace{-5pt}
Here $\mR^{\mathrm{co}}_{v,e}(\cdot; \va_{v, e})$ is a squared loss function following traditional hypergraph diffusion, and $\va_{v, e}$ can be any related attributes of the node-edge pair $(v, e)$ (\eg, node features or edge features). $\Omega^{\mathrm{co}}_{e}(\cdot)$ and $\Omega^{\mathrm{co}}_{v}(\cdot)$ are referred to as the co-edge and co-node regularization functions, respectively. They apply regularization to co-representations associated with the same node or edge, which can be implemented as any invariant structural regularization functions designed for traditional node-representation hypergraph diffusion \citep{zhou2007learning, hein2013total, hayhoe2023transferable}. Instead of making a choice from these handcrafted functions, in Section~\ref{sec:co_representation_neural_hypergraph_diffusion}, we will develop a neural implementation that can adaptively learn suitable diffusion dynamics from data. 

Depending on whether the regularization functions are differentiable, we can solve Eq.~\ref{eq:co_representation_hypergraph_diffusion} using one of two standard optimization methods: gradient descent (GD) or alternating direction method of multipliers (ADMM) \citep{boyd2011distributed}. We adopt the GD-based implementation throughout our experiments, while we also provide an ADMM-based implementation in our source code for completeness. We initialize $\vh_{v, e}^{(0)} = \va_{v, e}$, and solve it using GD with a step size $\alpha$: 
\begin{equation}\label{eq:gradient_descent}
\begin{split}
\vh^{(t+1)}_{v,e} =& 
\vh^{(t)}_{v,e} - 
\alpha \Big( \nabla \mR^{\mathrm{co}}_{v,e} \big( \vh^{(t)}_{v, e}; \va_{v, e} \big)\\
&+ 
\lambda \big[ \nabla \Omega^{\mathrm{co}}_{e} \big( \mH^{(t)}_{e} \big) \big]_{v} + 
\gamma \big[ \nabla \Omega^{\mathrm{co}}_{v} \big( \mH^{(t)}_{v} \big) \big]_{e} \Big),
\end{split}
\end{equation}
where $\nabla$ is the gradient operator. $[\cdot]_{v}$ and $[\cdot]_{e}$ represent the gradient vector associated with node $v$ and edge $e$, respectively. For example, $\big[ \nabla \Omega^{\mathrm{co}}_{e} (\mH^{(t)}_{e}) \big]_{v}$ represents the gradient w.r.t. $\vh^{(t)}_{v, e}$. Similar to the traditional hypergraph diffusion, we refer to $\nabla \Omega^{\mathrm{co}}_e(\cdot)$ as the \textit{co-edge diffusion operator}, which models within-edge interactions among co-representations and generates information that should ``diffuse'' to each node-edge pair. $\nabla \Omega^{\mathrm{co}}_v(\cdot)$ is referred to as the \textit{co-node diffusion operators}. We now reveal a critical property of the diffusion operators. 

\begin{prop}\label{prop:equivariant_diffusion_operators}
In the co-representation hypergraph diffusion with \textit{permutation invariant} co-edge and co-node regularization functions, the corresponding co-edge and co-node diffusion operators are \textit{permutation equivariant}. 
\end{prop}

\begin{proof}
We analyze $\nabla \Omega^{\mathrm{co}}_e(\cdot)$ here, while the analysis for $\nabla \Omega^{\mathrm{co}}_v(\cdot)$ is analogous. Since $\Omega^{\mathrm{co}}_e$ is permutation invariant, for any $\pi \in \mathbb{S}_n$ we have $\Omega^{\mathrm{co}}_e(\mP_\pi \mH) = \Omega^{\mathrm{co}}_e(\mH)$, where $\mP_{\pi}$ is the corresponding row permutation matrix of action $\pi$. Due to the linearity of the permutation action, the Jacobian matrix of $\mP_{\pi} \mH$ with respect to $\mH$ is $\mP_{\pi}$. By applying the chain rule we have
% \begin{equation}
% \nabla \Omega^{\mathrm{co}}_e ( \mP_{\pi} \mH ) = \frac{\partial \Omega^{\mathrm{co}}_e (\mP_{\pi} \mH)}{\partial (\mP_{\pi} \mH)} = \mP_{\pi} \frac{\partial \Omega^{\mathrm{co}}_e (\mP_{\pi} \mH)}{\partial (\mH)} = \mP_{\pi} \frac{\partial \Omega^{\mathrm{co}}_e (\mH)}{\partial (\mH)} = \mP_{\pi} \nabla \Omega^{\mathrm{co}}_e (\mH). 
% \nonumber
% \end{equation}
\begin{equation}
\nabla \Omega^{\mathrm{co}}_e ( \mP_{\pi} \mH ) = \mP_{\pi} \frac{\partial \Omega^{\mathrm{co}}_e (\mP_{\pi} \mH)}{\partial (\mH)} = \mP_{\pi} \frac{\partial \Omega^{\mathrm{co}}_e (\mH)}{\partial (\mH)} = \mP_{\pi} \nabla \Omega^{\mathrm{co}}_e (\mH). 
\nonumber
\end{equation}
This completes the proof. 
\end{proof}

This property shows that the diffusion operators derived from the co-representation diffusion process not only reformulate the within-edge and within-node interactions as multi-output functions, but also satisfy the equivariance property that ensures the diverse output results commute according to the input ordering. 

Next, we state the relation between the co-representation hypergraph diffusion process and the traditional node-representation hypergraph diffusion process. 

\begin{prop}\label{prop:connection_with_node_representation_hypergraph_diffusion}
The traditional node-representation hypergraph diffusion is a special case of the co-representation hypergraph diffusion, while the converse does not hold.
\end{prop}

\begin{proof}
For each $v \in \sV$, we introduce a set of auxiliary variables $\{ \vh_{v, e_i} | e_i \in \sEv \}$, satisfying $\vh_{v, e_i} = \vh_{v, e_j}$ for any $e_i, e_j \in \sEv$. Let \( \mH \) denote the collection of all auxiliary variables. Then the original problem in Eq.~\ref{eq:node_representation_hypergraph_diffusion} can be reformulated as the following constrained optimization problem:
\begin{equation}\label{eq:hypergraph_diffusion_with_constraints}
\begin{split}
\argmin_{\mH} \quad & \Bigg\{ \sum_{v \in \sV} \sum_{e \in \sEv} \frac{1}{d_v} \mR_{v}(\vh_{v, e}; \va_{v}) + \lambda \sum_{e \in \sE} \Omega_{e}(\mH_e) \Bigg\}, \\
\mathrm{s.t.} \quad & \vh_{v, e_i} = \vh_{v, e_j}, \quad \forall v \in \sV,\, \forall e_i, e_j \in \sEv.  
\end{split}
\end{equation}
Let \( \mH^\star \) be an optimal solution to Eq.~\eqref{eq:hypergraph_diffusion_with_constraints}. Then the solution to the original problem satisfies \( \vx_v^\star = \vh_{v,e}^\star \) for any \( e \in \sEv \).

We can set $\mR^{\mathrm{co}}_{v,e}(\cdot; \va_{v, e}) = \frac{1}{d_v} \mR_{v}(\cdot; \va_{v})$ and $\Omega^{\mathrm{co}}_{e} (\cdot) = \Omega_{e} (\cdot)$ in Eq.~\ref{eq:co_representation_hypergraph_diffusion}, and set $\Omega^{\mathrm{co}}_{v} (\cdot)$ as the CE regularization functions \citep{zhou2007learning}, \ie, $\Omega^{\mathrm{co}}_{v} \big( \mH_v \big) = \Omega_{\mathrm{CE}}(\mH_v) = \sum_{e_i, e_j \in \sEv} \lVert \vh_{v, e_i} - \vh_{v, e_j} \rVert_2^2$. Then Eq.~\ref{eq:co_representation_hypergraph_diffusion} can be reformulated as follows: 
\begin{equation}\label{eq:special_case_co_representation_hypergraph_diffusion}
\begin{split}
& \argmin_{\mH} \Bigg\{ \sum_{v \in \sV} \sum_{e \in \sEv} \frac{1}{d_v} \mR_{v}(\vh_{v, e}; \va_{v}) \\
& + \lambda \sum_{e \in \sE} \Omega_{e}(\mH_e) + \gamma \sum_{v \in \sV} \Omega_{\mathrm{CE}}(\mH_v) \Bigg\}. 
\end{split}
% \nonumber
\end{equation}
Here $\Omega_{\mathrm{CE}}(\cdot)$ is exactly the exterior penalty function \citep{yeniay2005penalty} for the given equality constraints in Eq.~\ref{eq:hypergraph_diffusion_with_constraints}. Thus as $\gamma \rightarrow \infty$, Eq.~\ref{eq:special_case_co_representation_hypergraph_diffusion} yields the same optimal solutions as Eq.~\ref{eq:hypergraph_diffusion_with_constraints}. The converse does not hold, since the node-representation hypergraph diffusion enforces a single node representation for each node and cannot accommodate multiple co-representations. This completes the proof.
\end{proof}

Node-representation hypergraph diffusion is equivalent to imposing a strict constraint that all the co-representations $\vh_{v, e_i}$ associated with the same node $v$ must be identical, resulting in a single unified node representation. Our method relaxes this hard constraint by co-node regularization functions, allowing multiple co-representations associated with the same node to be different while still being constrained by certain regularization terms.

%%%%%%%%%%%%%%%%%%%%%%%%%%%%%%%%%%%%%%%%%
% Co-representation Neural Hypergraph Diffusion
%%%%%%%%%%%%%%%%%%%%%%%%%%%%%%%%%%%%%%%%%
% \subsection{Co-representation Neural Hypergraph Diffusion}\label{sec:co_representation_neural_hypergraph_diffusion}
\subsection{Neural Implementation}\label{sec:co_representation_neural_hypergraph_diffusion}

In this section, we propose \underline{\textbf{Co}}-representation \underline{\textbf{N}}eural \underline{\textbf{H}}ypergraph \underline{\textbf{D}}iffusion (\textbf{CoNHD}), which is a neural implementation of the diffusion process defined in Definition~\ref{dfn:co_representation_hypergraph_diffusion} without the need for handcrafting regularization functions.

Since $\mR^{co}_{v,e}(\vh^{(t)}_{v, e}; \va_{v, e}) = \frac{1}{2} \lVert \vh^{(t)}_{v, e} - \va_{v, e} \rVert^2$ is a squared loss function, we have $\nabla \mR^{\mathrm{co}}_{v,e} ( \vh^{(t)}_{v, e}; \va_{v, e} ) = \vh^{(t)}_{v, e} -\va_{v, e}$. Eq.~\ref{eq:gradient_descent} can be rewritten as:  

\begin{equation}\label{eq:gradient_descent_extend}
\begin{split}
\vh^{(t+1)}_{v,e} =& (1-\alpha) \vh^{(t)}_{v,e} - \alpha \lambda \big[ \nabla \Omega^{\mathrm{co}}_{e} \big( \mH^{(t)}_{e} \big) \big]_{v} \\
&- \alpha \gamma \big[ \nabla \Omega^{\mathrm{co}}_{v} \big( \mH^{(t)}_{v} \big) \big]_{e} + \alpha \va_{v, e}.
\end{split}
\end{equation}
Therefore, $\vh^{(t+1)}_{v,e}$ is a linear combination of the co-representation in the last step $\vh^{(t)}_{v,e}$, within-edge and within-node diffusion information $\big[ \nabla \Omega^{\mathrm{co}}_{e} \big( \mH^{(t)}_{e} \big) \big]_{v}$ and $\big[ \nabla \Omega^{\mathrm{co}}_{v} \big( \mH^{(t)}_{v} \big) \big]_{e}$, and initial features $\vh^{(0)}_{v, e}=\va_{v, e}$. To avoid handcrafting regularization functions and manual choice of the factors $\alpha$, $\lambda$, and $\gamma$, we define two networks, $\phi$ and $\varphi$, to approximate the two interaction processes, and a linear layer $\psi$ to approximate the co-representation update process. The $(t+1)$-th layer can be represented as: 
\begin{align} 
\mM^{(t+1)}_e =& \phi \big( \mH^{(t)}_e \big),\, \mM'^{(t+1)}_v = \varphi \big( \mH^{(t)}_v \big), \label{eq:gd_based_1} \\ 
\vh^{(t+1)}_{v, e} =& \psi \big( \big[ \vh^{(t)}_{v, e}, \vm^{(t+1)}_{v, e}, \vm'^{(t+1)}_{v, e}. \vh^{(0)}_{v, e} \big] \big). \label{eq:gd_based_2}
\end{align}
Here, $\phi$ and $\varphi$ serve as the neural implementation of the diffusion operators, which should satisfy the permutation equivariance property of the co-edge and co-node diffusion operators stated in Proposition~\ref{prop:equivariant_diffusion_operators}. For the implementation of the diffusion operators, we explore two popular equivariant neural networks, UNB \citep{segol2020on, wang2023equivariant} and ISAB \citep{chien2022you}. Notably, CoNHD is a general HGNN framework allowing different equivariant network implementations for the diffusion operators, not limited to the two investigated in this work. $\mM^{(t)}_e = \Big[ \vm^{(t)}_{v^{e}_1, e}, \ldots, \vm^{(t)}_{v^{e}_{d_e}, e} \Big]^{\top}$ and $\mM'^{(t)}_v = \Big[ \vm'^{(t)}_{v, e^{v}_1}, \ldots, \vm'^{(t)}_{v, e^{v}_{d_v}} \Big]$ are the within-edge and within-node diffusion information generated using the neural diffusion operators $\phi$ and $\varphi$. The function $\psi(\cdot)$, implemented as a linear layer, collects diffusion information and updates the co-representations.

Previous research based on traditional hypergraph diffusion only explores modeling the composition of within-edge and within-node interactions as an equivariant function, while each interaction is still an invariant aggregation function \citep{wang2023equivariant}. Differently, WHATsNet \citep{choe2023classification} explores utilizing the equivariant module in both interactions. However, the multiple outputs serve only as intermediate results, with an aggregation module applied at the end, where the composition is still an invariant aggregation function and only a single node or edge representation is updated. In contrast, our method reformulates within-edge and within-node interactions as two distinct equivariant functions, ensuring diverse update information for different node-edge pairs, which is helpful for modeling node/edge-specific features and improves ENC performance as shown in Table~\ref{tab:ablation_study}. 

To demonstrate the expressiveness of CoNHD, we compare it with the message passing framework defined in Eq.~\ref{eq:message_passing_1}-\ref{eq:message_passing_2}, which, as previously noted, is followed by most HGNNs \citep{huang2021unignn, chien2022you}, including the state-of-the-art ENC solution WHATsNet \citep{choe2023classification}. Following \citep{choe2023classification}, we regard the concatenation of node and edge representations as the final embeddings, which can be used to predict ENC labels. This leads to the following proposition. 

\begin{prop}\label{prop:expressiveness}
With the same embedding dimension, CoNHD is expressive enough to represent the message passing framework, while the converse does not hold. 
\end{prop}

\begin{proof}
We provide a proof sketch here. We first show that CoNHD can express any models following the message passing framework. By initializing $\vh^{(0)}_{v,e} = [\vx^{(0)}_v, \vz^{(0)}_e]$, we can alternately update edge and node representations using two CoNHD layers. Specifically, the first aggregates 
nodes-to-edge messages via $\phi$ to form the outputs $\vh^{(2t+1)}_{v,e} = [\vx^{(t)}_v, \vz^{(t+1)}_e]$. The second aggregates edges-to-node messages via $\varphi$ to form the outputs $\vh^{(2(t+1))}_{v,e} = [\vx^{(t+1)}_v, \vz^{(t+1)}_e]$. Since $\phi$ and $\varphi$ can be implemented as universal equivariant functions like UNB, CoNHD can approximate the same updates as $\fVE$ and $\fEV$ in Eq.~\ref{eq:message_passing_1}-\ref{eq:message_passing_2}. Conversely, message passing models cannot express CoNHD, as they generate only one representation for each node or edge, whereas CoNHD allows multiple node-edge co-representations for each node or edge.
\end{proof}

Proposition~\ref{prop:expressiveness} shows that CoNHD is more expressive than all methods following the message passing framework. Notably, this gain in expressiveness does not incur additional complexity, as further analyzed in the next section.

\subsection{Complexity Analysis}\label{sec:complexity_analysis}

In this section, we analyze the time and space complexity of CoNHD compared to methods based on the dominant message passing framework, including WHATsNet.

\textbf{Time Complexity.} 
CoNHD computes within-edge and within-node interactions using UNB or ISAB operators, both with linear computational complexity relative to the input size. The complexity of both interactions is $\bigO ( \sum_{e \in \sE} ( d_e d^2 ) + \sum_{v \in \sV} ( d_v d^2 ) ) = \bigO ( d^2 \sum_{e \in \sE} d_e )$, where $d$ is the hidden size. For the update function, the complexity is $\bigO ( d^2 \sum_{e \in \sE} d_e )$. Therefore, the overall complexity of CoNHD is $\bigO ( L d^2 \sum_{e \in \sE} d_e )$, where $L$ is the number of layers. The overall time complexity is linear to the number of node-edge pairs, \ie, $\sum_{e \in \sE} d_e$, which is the same as other HGNNs within the message passing framework. For example, WHATsNet has the complexity of $\bigO(L \cdot (2 \cdot d^2 \sum_{e \in \sE} d_e + \sum_{e \in \mathcal{E}} d^2d_e + \sum_{v \in \mathcal{V}} d^2d_v))$, which is of the same order as that of CoNHD after ignoring constant factors and low-order terms. However, WHATsNet incurs higher runtime in practice due to additional edge-dependent feature extraction and aggregation steps in each layer, as well as the involvement of indirectly connected neighbors during the message passing process.

\textbf{Space Complexity.} Since the number of input co-representations in each layer of our model depends on the number of node-edge pairs, \ie, $\sum_{e \in \sE} d_e$, the size of the inputs is $\bigO ( d \sum_{e \in \sE} d_e )$. For within-edge and within-node interactions, both UNB and ISAB implementations utilize MLPs to perform feature transformation, where the size of the weights is $\bigO( d^2 )$. The sizes of the outputs for the within-edge and within-node interactions are $\bigO ( d \sum_{e \in \sE} d_e )$. In the update process, the concatenated input is a $4d$-dimensional vector. This is then passed through a linear layer to output the updated co-representations, where the weight size is $\bigO( 4d^2 )$. Therefore, the total space complexity of $L$ layers after removing the constant factors is $\bigO( L (d^2 + d \sum_{e \in \sE} d_e) ) = \bigO( Ld ( d + \sum_{e \in \sE} d_e ) )$. The overall space complexity is linear to the number of node-edge pairs in the input hypergraph, \ie, $\sum_{e \in \sE} d_e$. which is the same as those edge-dependent message passing methods, including WHATsNet \citep{choe2023classification}, which generate multiple edge-dependent node representations for each node in the calculation process.

% \subsection{Comparison to Existing Methods}\label{sec:comparison_to_existing_methods}

%%%%%%%%%%%%%%%%%%%%%%%%%%%%%%%%%%%%%%%%%%%%%%%%%%%%%%%%%%%%%%%%%%%%%%%%%%%%%%%
% Experiments
%%%%%%%%%%%%%%%%%%%%%%%%%%%%%%%%%%%%%%%%%%%%%%%%%%%%%%%%%%%%%%%%%%%%%%%%%%%%%%%
% \vspace{-1mm}
\section{Experiments}\label{sec:experiments}
% \vspace{-1mm}

\begin{table*}[t]
\setlength\tabcolsep{4pt}
\caption{\textbf{Performance of edge-dependent node classification.} \textbf{Bold} numbers represent the best results, while \underline{underlined} numbers indicate the second-best. ``O.O.M.'' means ``out of memory''. \setlength{\fboxsep}{1pt}\colorbox{lightgray}{Shaded cells} indicate that our method significantly outperforms the best baseline (p-value $<$ 0.05, based on the Wilcoxon signed-rank test). ``A.R.'' denotes the average ranking among all datasets.}\label{tab:ENC_performance}
\centering
% \vspace{-1mm}
\begin{adjustbox}{max width = 0.87\linewidth}
\begin{tabular}{c|cccccccccc|c}
\toprule
\multirow{2}{*}{Method} & \multicolumn{2}{c}{Email-Enron} & \multicolumn{2}{c}{Email-Eu} & \multicolumn{2}{c}{Stack-Biology} & \multicolumn{2}{c}{Stack-Physics} & \multicolumn{2}{c|}{Coauth-DBLP} & \multirow{2}{*}{\shortstack[c]{A.R. of\\Micro-F1}} \\
% \cmidrule(lr){2-3} \cmidrule(lr){4-5} \cmidrule(lr){6-7} \cmidrule(lr){8-9} \cmidrule(lr){10-11}
& Micro-F1 & Macro-F1 & Micro-F1 & Macro-F1 & Micro-F1 & Macro-F1 & Micro-F1 & Macro-F1 & Micro-F1 & Macro-F1 & \\ 
\midrule
GraphSAGE & 0.775 \scriptsize{$\pm$ 0.005} & 0.714 \scriptsize{$\pm$ 0.007} & 0.658 \scriptsize{$\pm$ 0.001} & 0.564 \scriptsize{$\pm$ 0.005} & 0.689 \scriptsize{$\pm$ 0.010} & 0.598 \scriptsize{$\pm$ 0.014} & 0.660 \scriptsize{$\pm$ 0.011} & 0.523 \scriptsize{$\pm$ 0.018} & 0.474 \scriptsize{$\pm$ 0.002} & 0.401 \scriptsize{$\pm$ 0.008} & 11.7 \\
GAT & 0.736 \scriptsize{$\pm$ 0.056} & 0.611 \scriptsize{$\pm$ 0.103} & 0.618 \scriptsize{$\pm$ 0.002} & 0.580 \scriptsize{$\pm$ 0.024} & 0.692 \scriptsize{$\pm$ 0.015} & 0.628 \scriptsize{$\pm$ 0.010} & 0.725 \scriptsize{$\pm$ 0.024} & 0.636 \scriptsize{$\pm$ 0.043} & 0.575 \scriptsize{$\pm$ 0.005} & 0.558 \scriptsize{$\pm$ 0.007} & 8.1 \\
ADGN & 0.790 \scriptsize{$\pm$ 0.001} & 0.723 \scriptsize{$\pm$ 0.001} & 0.667 \scriptsize{$\pm$ 0.001} & 0.622 \scriptsize{$\pm$ 0.006} & 0.714 \scriptsize{$\pm$ 0.002} & 0.651 \scriptsize{$\pm$ 0.001} & 0.686 \scriptsize{$\pm$ 0.014} & 0.537 \scriptsize{$\pm$ 0.019} & 0.505 \scriptsize{$\pm$ 0.006} & 0.440 \scriptsize{$\pm$ 0.020} & 8.6 \\
\midrule
HyperGNN & 0.725 \scriptsize{$\pm$ 0.004} & 0.674 \scriptsize{$\pm$ 0.003} & 0.633 \scriptsize{$\pm$ 0.001} & 0.533 \scriptsize{$\pm$ 0.008} & 0.689 \scriptsize{$\pm$ 0.002} & 0.624 \scriptsize{$\pm$ 0.007} & 0.686 \scriptsize{$\pm$ 0.004} & 0.630 \scriptsize{$\pm$ 0.002} & 0.540 \scriptsize{$\pm$ 0.004} & 0.519 \scriptsize{$\pm$ 0.002} & 10.0 \\
HAT & 0.817 \scriptsize{$\pm$ 0.001} & 0.753 \scriptsize{$\pm$ 0.004} & 0.669 \scriptsize{$\pm$ 0.001} & 0.638 \scriptsize{$\pm$ 0.002} & 0.661 \scriptsize{$\pm$ 0.005} & 0.606 \scriptsize{$\pm$ 0.005} & 0.708 \scriptsize{$\pm$ 0.005} & 0.643 \scriptsize{$\pm$ 0.009} & 0.503 \scriptsize{$\pm$ 0.004} & 0.483 \scriptsize{$\pm$ 0.006} & 7.9 \\
UniGCNII & 0.734 \scriptsize{$\pm$ 0.010} & 0.656 \scriptsize{$\pm$ 0.010} & 0.630 \scriptsize{$\pm$ 0.005} & 0.565 \scriptsize{$\pm$ 0.013} & 0.610 \scriptsize{$\pm$ 0.004} & 0.433 \scriptsize{$\pm$ 0.007} & 0.671 \scriptsize{$\pm$ 0.022} & 0.492 \scriptsize{$\pm$ 0.016} & 0.497 \scriptsize{$\pm$ 0.003} & 0.476 \scriptsize{$\pm$ 0.002} & 13.5 \\
AllSet & 0.796 \scriptsize{$\pm$ 0.014} & 0.719 \scriptsize{$\pm$ 0.020} & 0.666 \scriptsize{$\pm$ 0.005} & 0.624 \scriptsize{$\pm$ 0.021} & 0.571 \scriptsize{$\pm$ 0.054} & 0.446 \scriptsize{$\pm$ 0.081} & 0.728 \scriptsize{$\pm$ 0.039} & 0.646 \scriptsize{$\pm$ 0.046} & 0.495 \scriptsize{$\pm$ 0.038} & 0.487 \scriptsize{$\pm$ 0.040} & 9.0 \\
HDS$^{ode}$ & 0.805 \scriptsize{$\pm$ 0.001} & 0.740 \scriptsize{$\pm$ 0.006} & 0.651 \scriptsize{$\pm$ 0.000} & 0.577 \scriptsize{$\pm$ 0.005} & 0.708 \scriptsize{$\pm$ 0.001} & 0.643 \scriptsize{$\pm$ 0.004} & 0.737 \scriptsize{$\pm$ 0.001} & 0.635 \scriptsize{$\pm$ 0.008} & 0.558 \scriptsize{$\pm$ 0.001} & 0.550 \scriptsize{$\pm$ 0.002} & 7.3 \\ 
LEGCN & 0.783 \scriptsize{$\pm$ 0.001} & 0.728 \scriptsize{$\pm$ 0.007} & 0.639 \scriptsize{$\pm$ 0.001} & 0.535 \scriptsize{$\pm$ 0.004} & 0.668 \scriptsize{$\pm$ 0.002} & 0.572 \scriptsize{$\pm$ 0.006} & 0.701 \scriptsize{$\pm$ 0.003} & 0.575 \scriptsize{$\pm$ 0.018} & 0.499 \scriptsize{$\pm$ 0.003} & 0.490 \scriptsize{$\pm$ 0.002} & 7.5 \\
MultiSetMixer & 0.818 \scriptsize{$\pm$ 0.001} & 0.755 \scriptsize{$\pm$ 0.005} & 0.670 \scriptsize{$\pm$ 0.001} & 0.636 \scriptsize{$\pm$ 0.005} & 0.709 \scriptsize{$\pm$ 0.001} & 0.643 \scriptsize{$\pm$ 0.003} & 0.754 \scriptsize{$\pm$ 0.001} & 0.679 \scriptsize{$\pm$ 0.004} & 0.559 \scriptsize{$\pm$ 0.001} & 0.554 \scriptsize{$\pm$ 0.001} & 6.0 \\
HNN & 0.763 \scriptsize{$\pm$ 0.003} & 0.679 \scriptsize{$\pm$ 0.007} & O.O.M. & O.O.M. & 0.618 \scriptsize{$\pm$ 0.015} & 0.568 \scriptsize{$\pm$ 0.013} & 0.683 \scriptsize{$\pm$ 0.005} & 0.617 \scriptsize{$\pm$ 0.005} & 0.488 \scriptsize{$\pm$ 0.006} & 0.482 \scriptsize{$\pm$ 0.006} & 12.4 \\ 
ED-HNN & 0.778 \scriptsize{$\pm$ 0.001} & 0.713 \scriptsize{$\pm$ 0.004} & 0.648 \scriptsize{$\pm$ 0.001} & 0.558 \scriptsize{$\pm$ 0.004} & 0.688 \scriptsize{$\pm$ 0.005} & 0.506 \scriptsize{$\pm$ 0.002} & 0.726 \scriptsize{$\pm$ 0.002} & 0.617 \scriptsize{$\pm$ 0.006} & 0.514 \scriptsize{$\pm$ 0.016} & 0.484 \scriptsize{$\pm$ 0.024} & 9.3 \\
WHATsNet & 0.826 \scriptsize{$\pm$ 0.001} & 0.761 \scriptsize{$\pm$ 0.003} & 0.671 \scriptsize{$\pm$ 0.000} & 0.645 \scriptsize{$\pm$ 0.003} & 0.742 \scriptsize{$\pm$ 0.002} & 0.685 \scriptsize{$\pm$ 0.003} & 0.770 \scriptsize{$\pm$ 0.003} & 0.707 \scriptsize{$\pm$ 0.004} & 0.604 \scriptsize{$\pm$ 0.003} & 0.592 \scriptsize{$\pm$ 0.004} & 5.2 \\
\midrule
CoNHD (UNB) (\textit{ours}) & \cellcolor{lightgray} \underline{0.905 \scriptsize{$\pm$ 0.001}} & \cellcolor{lightgray} \underline{0.858 \scriptsize{$\pm$ 0.004}} & \cellcolor{lightgray} \underline{0.708 \scriptsize{$\pm$ 0.001}} & \cellcolor{lightgray} \underline{0.689 \scriptsize{$\pm$ 0.001}} & \cellcolor{lightgray} \underline{0.748 \scriptsize{$\pm$ 0.003}} & \cellcolor{lightgray} \underline{0.694 \scriptsize{$\pm$ 0.005}} & \cellcolor{lightgray} \underline{0.776 \scriptsize{$\pm$ 0.001}} & \textbf{0.712 \scriptsize{$\pm$ 0.005}} & \cellcolor{lightgray} \textbf{0.620 \scriptsize{$\pm$ 0.002}} & \cellcolor{lightgray} \textbf{0.604 \scriptsize{$\pm$ 0.002}} & \underline{1.9} \\
CoNHD (ISAB) (\textit{ours}) & \cellcolor{lightgray} \textbf{0.911 \scriptsize{$\pm$ 0.001}} & \cellcolor{lightgray} \textbf{0.871 \scriptsize{$\pm$ 0.002}} & \cellcolor{lightgray} \textbf{0.709 \scriptsize{$\pm$ 0.001}} & \cellcolor{lightgray} \textbf{0.690 \scriptsize{$\pm$ 0.002}} & \cellcolor{lightgray} \textbf{0.749 \scriptsize{$\pm$ 0.002}} & \cellcolor{lightgray} \textbf{0.695 \scriptsize{$\pm$ 0.004}} & \cellcolor{lightgray} \textbf{0.777 \scriptsize{$\pm$ 0.001}} & \underline{0.710 \scriptsize{$\pm$ 0.004}} & \cellcolor{lightgray} \underline{0.619 \scriptsize{$\pm$ 0.002}} & \cellcolor{lightgray} \textbf{0.604 \scriptsize{$\pm$ 0.003}} & \textbf{1.1} \\
\bottomrule 
\toprule 
\multirow{2}{*}{Method} & \multicolumn{2}{c}{Coauth-AMiner} & \multicolumn{2}{c}{Cora-Outsider} & \multicolumn{2}{c}{DBLP-Outsider} & \multicolumn{2}{c}{Citeseer-Outsider} & \multicolumn{2}{c|}{Pubmed-Outsider} & \multirow{2}{*}{\shortstack[c]{A.R. of\\Macro-F1}} \\
% \cmidrule(lr){2-3} \cmidrule(lr){4-5} \cmidrule(lr){6-7} \cmidrule(lr){8-9} \cmidrule(lr){10-11}
& Micro-F1 & Macro-F1 & Micro-F1 & Macro-F1 & Micro-F1 & Macro-F1 & Micro-F1 & Macro-F1 & Micro-F1 & Macro-F1 \\ 
\midrule
GraphSAGE & 0.441 \scriptsize{$\pm$ 0.013} & 0.398 \scriptsize{$\pm$ 0.012} & 0.520 \scriptsize{$\pm$ 0.009} & 0.518 \scriptsize{$\pm$ 0.007} & 0.490 \scriptsize{$\pm$ 0.029} & 0.427 \scriptsize{$\pm$ 0.083} & 0.704 \scriptsize{$\pm$ 0.005} & 0.704 \scriptsize{$\pm$ 0.005} & 0.677 \scriptsize{$\pm$ 0.003} & 0.663 \scriptsize{$\pm$ 0.002} & 11.8 \\
GAT & 0.623 \scriptsize{$\pm$ 0.006} & 0.608 \scriptsize{$\pm$ 0.009} & 0.531 \scriptsize{$\pm$ 0.009} & 0.521 \scriptsize{$\pm$ 0.008} & 0.563 \scriptsize{$\pm$ 0.003} & 0.548 \scriptsize{$\pm$ 0.003} & 0.704 \scriptsize{$\pm$ 0.011} & 0.702 \scriptsize{$\pm$ 0.011} & 0.677 \scriptsize{$\pm$ 0.003} & 0.670 \scriptsize{$\pm$ 0.002} & 7.9 \\
ADGN & 0.452 \scriptsize{$\pm$ 0.009} & 0.415 \scriptsize{$\pm$ 0.014} & 0.533 \scriptsize{$\pm$ 0.007} & 0.524 \scriptsize{$\pm$ 0.005} & 0.559 \scriptsize{$\pm$ 0.005} & 0.548 \scriptsize{$\pm$ 0.001} & 0.706 \scriptsize{$\pm$ 0.008} & 0.705 \scriptsize{$\pm$ 0.008} & 0.669 \scriptsize{$\pm$ 0.003} & 0.667 \scriptsize{$\pm$ 0.002} & 9.0 \\
\midrule 
HyperGNN & 0.566 \scriptsize{$\pm$ 0.002} & 0.551 \scriptsize{$\pm$ 0.004} & 0.532 \scriptsize{$\pm$ 0.015} & 0.528 \scriptsize{$\pm$ 0.013} & 0.571 \scriptsize{$\pm$ 0.005} & 0.566 \scriptsize{$\pm$ 0.005} & 0.696 \scriptsize{$\pm$ 0.006} & 0.696 \scriptsize{$\pm$ 0.006} & 0.658 \scriptsize{$\pm$ 0.003} & 0.654 \scriptsize{$\pm$ 0.002} & 9.5 \\
HAT & 0.543 \scriptsize{$\pm$ 0.002} & 0.533 \scriptsize{$\pm$ 0.003} & 0.548 \scriptsize{$\pm$ 0.015} & 0.544 \scriptsize{$\pm$ 0.017} & 0.588 \scriptsize{$\pm$ 0.002} & 0.586 \scriptsize{$\pm$ 0.002} & 0.691 \scriptsize{$\pm$ 0.018} & 0.690 \scriptsize{$\pm$ 0.019} & 0.676 \scriptsize{$\pm$ 0.003} & 0.673 \scriptsize{$\pm$ 0.003} & 6.8 \\
UniGCNII & 0.520 \scriptsize{$\pm$ 0.001} & 0.507 \scriptsize{$\pm$ 0.001} & 0.519 \scriptsize{$\pm$ 0.019} & 0.509 \scriptsize{$\pm$ 0.023} & 0.540 \scriptsize{$\pm$ 0.004} & 0.537 \scriptsize{$\pm$ 0.006} & 0.674 \scriptsize{$\pm$ 0.018} & 0.671 \scriptsize{$\pm$ 0.023} & 0.621 \scriptsize{$\pm$ 0.004} & 0.617 \scriptsize{$\pm$ 0.006} & 13.3 \\
AllSet & 0.577 \scriptsize{$\pm$ 0.005} & 0.570 \scriptsize{$\pm$ 0.002} & 0.523 \scriptsize{$\pm$ 0.018} & 0.502 \scriptsize{$\pm$ 0.016} & 0.585 \scriptsize{$\pm$ 0.008} & 0.515 \scriptsize{$\pm$ 0.013} & 0.686 \scriptsize{$\pm$ 0.010} & 0.681 \scriptsize{$\pm$ 0.009} & 0.679 \scriptsize{$\pm$ 0.006} & 0.660 \scriptsize{$\pm$ 0.010} & 10.1 \\
HDS$^{ode}$ & 0.561 \scriptsize{$\pm$ 0.003} & 0.552 \scriptsize{$\pm$ 0.003} & 0.537 \scriptsize{$\pm$ 0.009} & 0.529 \scriptsize{$\pm$ 0.010} & 0.554 \scriptsize{$\pm$ 0.004} & 0.548 \scriptsize{$\pm$ 0.002} & 0.703 \scriptsize{$\pm$ 0.008} & 0.703 \scriptsize{$\pm$ 0.008} & 0.669 \scriptsize{$\pm$ 0.004} & 0.664 \scriptsize{$\pm$ 0.005} & 7.1 \\
LEGCN & 0.520 \scriptsize{$\pm$ 0.002} & 0.511 \scriptsize{$\pm$ 0.003} & 0.698 \scriptsize{$\pm$ 0.008} & 0.689 \scriptsize{$\pm$ 0.008} & 0.676 \scriptsize{$\pm$ 0.016} & 0.675 \scriptsize{$\pm$ 0.016} & 0.733 \scriptsize{$\pm$ 0.015} & 0.731 \scriptsize{$\pm$ 0.016} & 0.703 \scriptsize{$\pm$ 0.002} & 0.698 \scriptsize{$\pm$ 0.002} & 7.4 \\
MultiSetMixer & 0.593 \scriptsize{$\pm$ 0.005} & 0.585 \scriptsize{$\pm$ 0.005} & 0.542 \scriptsize{$\pm$ 0.013} & 0.538 \scriptsize{$\pm$ 0.011} & 0.561 \scriptsize{$\pm$ 0.004} & 0.552 \scriptsize{$\pm$ 0.003} & 0.706 \scriptsize{$\pm$ 0.007} & 0.705 \scriptsize{$\pm$ 0.007} & 0.668 \scriptsize{$\pm$ 0.001} & 0.666 \scriptsize{$\pm$ 0.001} & 5.6 \\
HNN & 0.543 \scriptsize{$\pm$ 0.002} & 0.533 \scriptsize{$\pm$ 0.002} & 0.522 \scriptsize{$\pm$ 0.008} & 0.354 \scriptsize{$\pm$ 0.008} & 0.527 \scriptsize{$\pm$ 0.006} & 0.409 \scriptsize{$\pm$ 0.083} & 0.527 \scriptsize{$\pm$ 0.028} & 0.436 \scriptsize{$\pm$ 0.094} & 0.673 \scriptsize{$\pm$ 0.006} & 0.668 \scriptsize{$\pm$ 0.006} & 11.9 \\
ED-HNN & 0.503 \scriptsize{$\pm$ 0.006} & 0.479 \scriptsize{$\pm$ 0.008} & 0.532 \scriptsize{$\pm$ 0.011} & 0.511 \scriptsize{$\pm$ 0.014} & 0.599 \scriptsize{$\pm$ 0.002} & 0.559 \scriptsize{$\pm$ 0.013} & 0.709 \scriptsize{$\pm$ 0.007} & 0.709 \scriptsize{$\pm$ 0.007} & 0.668 \scriptsize{$\pm$ 0.008} & 0.656 \scriptsize{$\pm$ 0.009} & 11.1 \\ 
WHATsNet & 0.632 \scriptsize{$\pm$ 0.004} & 0.625 \scriptsize{$\pm$ 0.006} & 0.526 \scriptsize{$\pm$ 0.014} & 0.519 \scriptsize{$\pm$ 0.014} & 0.587 \scriptsize{$\pm$ 0.004} & 0.582 \scriptsize{$\pm$ 0.008} & 0.711 \scriptsize{$\pm$ 0.010} & 0.710 \scriptsize{$\pm$ 0.009} & 0.677 \scriptsize{$\pm$ 0.004} & 0.670 \scriptsize{$\pm$ 0.004} & 4.9 \\ 
\midrule
CoNHD (UNB) (\textit{ours}) & \cellcolor{lightgray} \underline{0.646 \scriptsize{$\pm$ 0.003}} & \cellcolor{lightgray} \underline{0.640 \scriptsize{$\pm$ 0.004}} & \cellcolor{lightgray} \underline{0.769 \scriptsize{$\pm$ 0.028}} & \cellcolor{lightgray} \underline{0.767 \scriptsize{$\pm$ 0.028}} & \cellcolor{lightgray} \underline{0.884 \scriptsize{$\pm$ 0.011}} & \cellcolor{lightgray} \underline{0.883 \scriptsize{$\pm$ 0.011}} & \cellcolor{lightgray} \underline{0.827 \scriptsize{$\pm$ 0.013}} & \cellcolor{lightgray} \textbf{0.826 \scriptsize{$\pm$ 0.013}} & \cellcolor{lightgray} \underline{0.896 \scriptsize{$\pm$ 0.003}} & \cellcolor{lightgray} \underline{0.895 \scriptsize{$\pm$ 0.003}} & \underline{1.8} \\
CoNHD (ISAB) (\textit{ours}) & \cellcolor{lightgray} \textbf{0.650 \scriptsize{$\pm$ 0.003}} & \cellcolor{lightgray} \textbf{0.646 \scriptsize{$\pm$ 0.004}} & \cellcolor{lightgray} \textbf{0.800 \scriptsize{$\pm$ 0.019}} & \cellcolor{lightgray} \textbf{0.797 \scriptsize{$\pm$ 0.020}} & \cellcolor{lightgray} \textbf{0.903 \scriptsize{$\pm$ 0.002}} & \cellcolor{lightgray} \textbf{0.902 \scriptsize{$\pm$ 0.002}} & \cellcolor{lightgray} \textbf{0.828 \scriptsize{$\pm$ 0.010}} & \cellcolor{lightgray} \textbf{0.826 \scriptsize{$\pm$ 0.010}} & \cellcolor{lightgray} \textbf{0.899 \scriptsize{$\pm$ 0.004}} & \cellcolor{lightgray} \textbf{0.898 \scriptsize{$\pm$ 0.004}} & \textbf{1.1} \\
\bottomrule 
\end{tabular}
\end{adjustbox}
% \vspace{-6mm}
\end{table*}

In this section, we conduct experiments to evaluate the effectiveness and efficiency of the proposed CoNHD method on the ENC task as well as several downstream tasks.

% \vspace{-2mm}
\subsection{Effectiveness and Efficiency on ENC}\label{sec:exp_ENC}
% \vspace{-1mm}

\textbf{Datasets.} We conduct experiments on ten ENC datasets. These datasets include all six datasets in \citep{choe2023classification}, which are Email (Email-Enron and Email-Eu), StackOverflow (Stack-Biology and Stack-Physics), and Co-authorship networks (Coauth-DBLP and Coauth-AMiner). Notably, Email-Enron and Email-Eu have relatively large node degrees, while Email-Enron has relatively large edge degrees as well. Additionally, as real-world hypergraph structures typically contain more noise \citep{cai2022hypergraph} than benchmark datasets, we construct four new datasets (Cora-Outsider, DBLP-Outsider, Citeseer-Outsider, and Pubmed-Outsider) by transforming the outsider identification task \citep{zhang2020hyper} into the ENC task. In these datasets, we randomly replace half of the nodes in each edge with other nodes, and the task is to predict whether each node belongs to the corresponding edge.

\textbf{Baselines.} We compare CoNHD (with two diffusion operator implementations, UNB and ISAB) to ten baseline HGNN methods, including five models following the traditional message passing framework (HyperGNN \citep{feng2019hypergraph}, HAT \citep{hwang2021hyfer}, UniGCNII \citep{huang2021unignn}, AllSet \citep{chien2022you}, and HDS$^{ode}$ \citep{yan2024hypergraph}) and five models that utilize edge-dependent node information (LEGCN \citep{yang2022semi}, MultiSetMixer \citep{telyatnikov2023hypergraph}, HNN \citep{aponte2022hypergraph}, ED-HNN \citep{wang2023equivariant}, and WHATsNet \citep{choe2023classification}). Since a hypergraph can also be viewed as a bipartite graph, we add three traditional GNN methods (GraphSAGE \citep{hamilton2017inductive}, GAT \citep{velivckovic2018graph}, and a graph diffusion-based method ADGN \citep{gravinaanti2023anti}) as our baselines. We follow the experiment setup in \citep{choe2023classification}.

% More details are provided in Appendix~\ref{app:implementation_details}. 

\textbf{Effectiveness.} As shown in Table~\ref{tab:ENC_performance}, CoNHD achieves the best performance across all datasets in both Micro-F1 and Macro-F1. Notably, CoNHD shows substantial improvements on Email-Enron and Email-Eu, which have relatively large-degree nodes or edges. Baseline methods using single node or edge representations can cause information loss for large degree nodes or edges in these datasets. In contrast, the number of co-representations in CoNHD is adaptive to the node and edge degrees. Additionally, CoNHD achieves very significant improvements on four outsider identification datasets. This suggests that mixing features from noise outsiders into an entangled edge representation significantly degrades the performance. Our method, with the co-representation design, can differentiate features from normal nodes and outsiders, leading to superior results. On simpler datasets with very low node and edge degrees, such as Stack-Physics, the improvement is less pronounced. In these datasets, each hyperedge only contains a very limited number of nodes (about 2 on average, similar to normal graphs) and cannot fully demonstrate the ability of different HGNNs in modeling complex higher-order interactions. Nevertheless, our method still consistently achieves the best performance on these datasets with statistically significant improvement (p-value $<$ 0.5) in most cases. The performance gap between the two diffusion operator implementations (UNB and ISAB) is minimal, while the transformer-based ISAB implementation overall demonstrates better performance. Unless otherwise specified, we adopt the better ISAB implementation for most subsequent experiments.

To quantify the diversity of ENC labels of the same node and explore its impact on model performance, we introduce a measure called node entropy. Specifically, for a node $v$, the node entropy $H(v)$ is defined as $H(v) := - \sum_{c \in \sC} p_v(c) \log p_v(c)$, where $p_v(c) = \sum_{e \in \sEv} \mathbb{I}(y_{v,e} = c) / d_v$. A higher node entropy value indicates that the node labels in different edges tend to be different. The edge entropy is defined similarly. We sort the entropy values in ascending order and divide them into ten equal-frequency levels. We calculate the model performance in each level. As shown in Fig.~\ref{fig:emailenron_entropy_microf1}, CoNHD demonstrates advantages compared to message passing methods like WHATsNet as the entropy level increases. This suggests that using single node representations is insufficient to capture different node/edge-specific features for predicting different ENC labels.

\begin{figure}[t]
\centering
% \vspace{-10pt}
\includegraphics[width=1\linewidth]{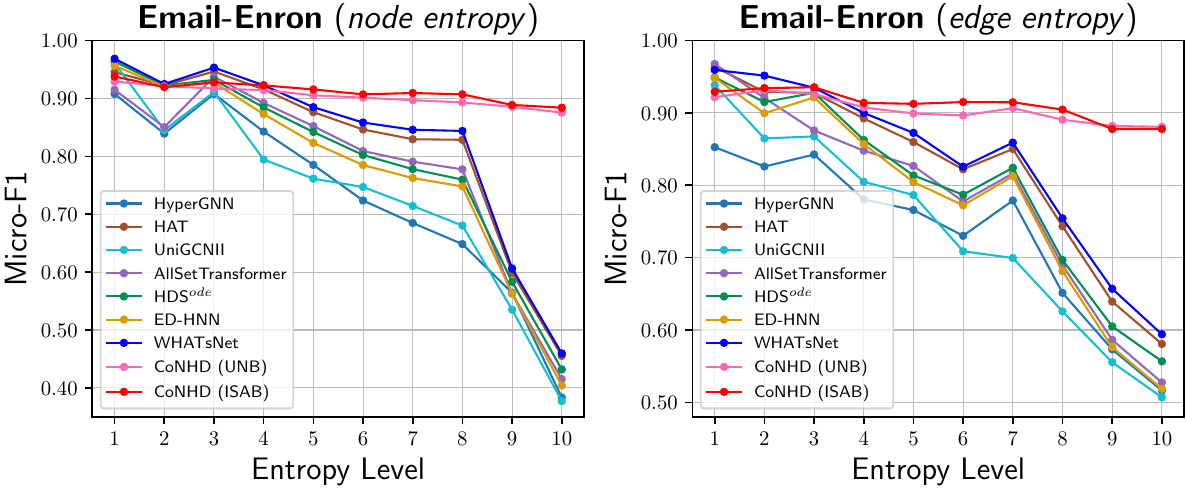}
% \vspace{-18pt}
\caption{\textbf{Comparison of performance under different node/edge entropy levels.} As node/edge entropy increases, the performance of message passing methods drops significantly, whereas CoNHD still maintains high performance.}
\label{fig:emailenron_entropy_microf1}
% \vspace{-14pt}
\end{figure}

\begin{figure}[t]
\centering
% \vspace{-4pt}
\includegraphics[width=1.0\linewidth]{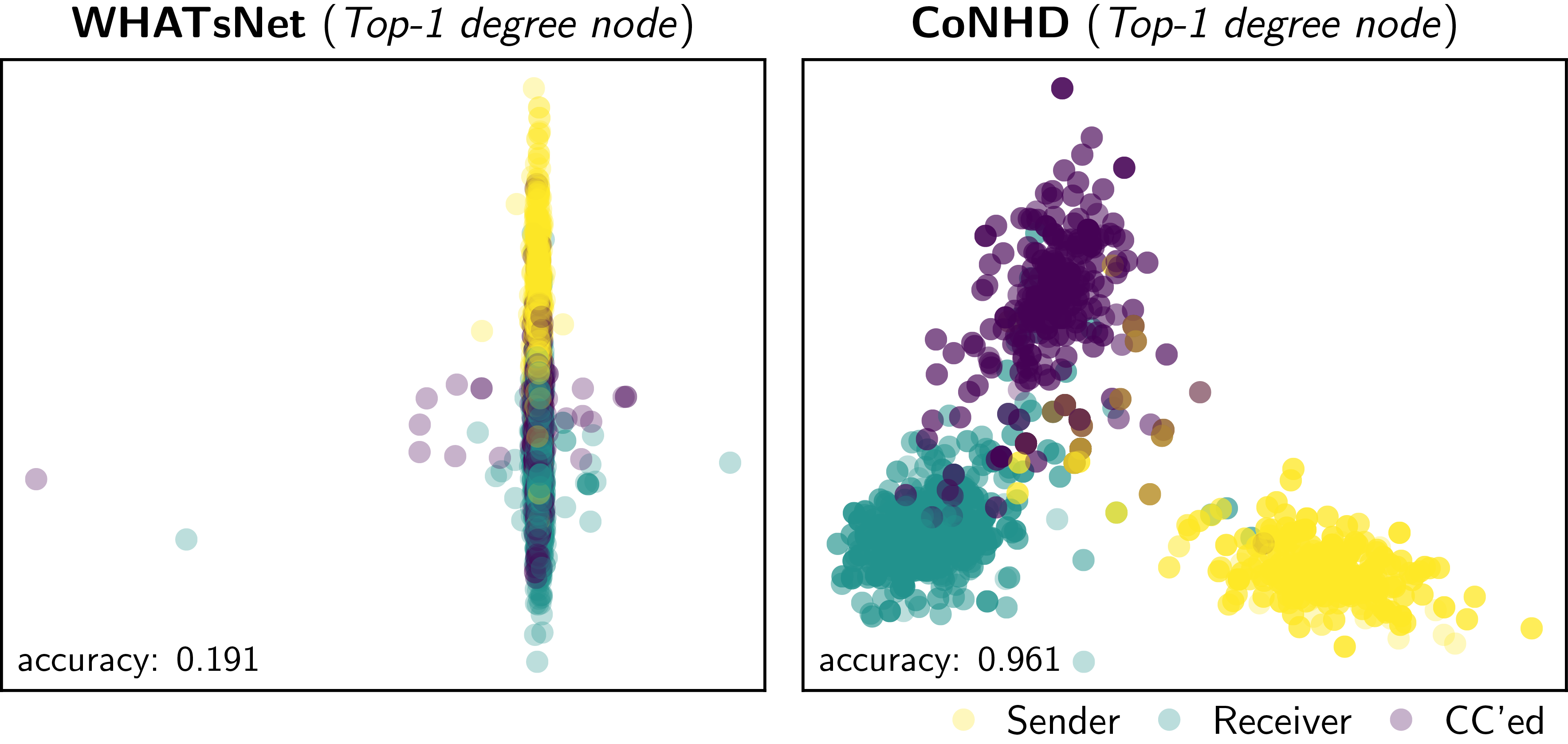}
% \vspace{-24pt}
\caption{\textbf{Visualization of embeddings of the top-1 degree node (user) within different hyperedges (emails) in Email-Enron using LDA.} The embeddings learned by CoNHD exhibit clearer distinctions based on labels in different hyperedges compared to the embeddings learned by WHATsNet. }
% \vspace{-10pt}
\label{fig:visualization_embeddings}
\end{figure}

\begin{table*}[t]
\setlength\tabcolsep{3pt}
\caption{\textbf{Performance of predicting ENC labels on downstream datasets.}}\label{tab:ENC_performance_downstream}
\centering
% \vspace{-8pt}
\begin{adjustbox}{max width = 0.68\linewidth}
\begin{tabular}{c|cccccccc}
\toprule
\multirow{2}{*}{Method} & \multicolumn{2}{c}{Halo} & \multicolumn{2}{c}{AMiner} & \multicolumn{2}{c}{DBLP} & \multicolumn{2}{c}{Etail} \\
% \cmidrule(lr){2-3} \cmidrule(lr){4-5} \cmidrule(lr){6-7} \cmidrule(lr){8-9} 
& Micro-F1 & Macro-F1 & Micro-F1 & Macro-F1 & Micro-F1 & Macro-F1 & Micro-F1 & Macro-F1 \\ 
\midrule
WHATsNet & 0.377 \scriptsize{$\pm$ 0.002} & 0.352 \scriptsize{$\pm$ 0.006} & 0.631 \scriptsize{$\pm$ 0.027} & 0.561 \scriptsize{$\pm$ 0.044} & 0.625 \scriptsize{$\pm$ 0.092} & 0.553 \scriptsize{$\pm$ 0.128} & 0.622 \scriptsize{$\pm$ 0.004} & 0.461 \scriptsize{$\pm$ 0.007} \\
CoNHD (\textit{ours}) & \cellcolor{lightgray} \textbf{0.396 \scriptsize{$\pm$ 0.003}} & \cellcolor{lightgray} \textbf{0.381 \scriptsize{$\pm$ 0.007}} & \textbf{0.661 \scriptsize{$\pm$ 0.027}} & \textbf{0.605 \scriptsize{$\pm$ 0.040}} & \cellcolor{lightgray} \textbf{0.768 \scriptsize{$\pm$ 0.094}} & \cellcolor{lightgray} \textbf{0.740 \scriptsize{$\pm$ 0.127}} & \cellcolor{lightgray} \textbf{0.751 \scriptsize{$\pm$ 0.008}} & \cellcolor{lightgray} \textbf{0.696 \scriptsize{$\pm$ 0.008}} \\
\bottomrule 
\end{tabular}
\end{adjustbox}
% \vspace{-4mm}
\end{table*}

\begin{table*}[t]
\setlength\tabcolsep{3pt}
\caption{\textbf{Performance on downstream tasks using the predicted ENC labels. }}\label{tab:performance_downstream}
\centering
\vspace{-8pt}
\begin{adjustbox}{max width =1.0\linewidth}
\begin{subtable}{.43\linewidth}
\centering
% \captionsetup{font=large,labelfont=large}
\captionsetup{font=large}
\caption{Ranking Aggregation (Acc.$\uparrow$)}
\begin{tabular}{l|cc}
    \toprule
    Method & Halo & AMiner \\ 
    \midrule
    RW \cite{chitra2019random} w/o Labels & 0.532 \scriptsize{$\pm$ 0.000} & 0.654 \scriptsize{$\pm$ 0.000} \\
    \midrule
    RW \cite{chitra2019random} w/ WHATsNet & 0.714 \scriptsize{$\pm$ 0.004} & 0.693 \scriptsize{$\pm$ 0.001} \\
    RW \cite{chitra2019random} w/ CoNHD & \cellcolor{lightgray} \textbf{0.723 \scriptsize{$\pm$ 0.003}} & \cellcolor{lightgray} \textbf{0.695 \scriptsize{$\pm$ 0.001}} \\
    \midrule
    RW \cite{chitra2019random} w/ GroundTruth & 0.711 \scriptsize{$\pm$ 0.000} & 0.675 \scriptsize{$\pm$ 0.000} \\
    \bottomrule
\end{tabular}
\end{subtable}
% \hspace{-4pt}
\begin{subtable}{.43\linewidth}
\centering
% \captionsetup{font=large,labelfont=large}
\captionsetup{font=large}
\caption{Clustering (NMI$\uparrow$)}
\begin{tabular}{l|cc}
    \toprule
    Method & DBLP & AMiner \\
    \midrule
    RDC-Spec \cite{hayashi2020hypergraph} w/o Labels & 0.163 \scriptsize{$\pm$ 0.000} & 0.338 \scriptsize{$\pm$ 0.000} \\ 
    \midrule
    RDC-Spec \cite{hayashi2020hypergraph} w/ WHATsNet & 0.184 \scriptsize{$\pm$ 0.028} & 0.352 \scriptsize{$\pm$ 0.034} \\
    RDC-Spec \cite{hayashi2020hypergraph} w/ CoNHD & \textbf{0.196 \scriptsize{$\pm$ 0.022}} & \textbf{0.354 \scriptsize{$\pm$ 0.016}} \\
    \midrule
    RDC-Spec \cite{hayashi2020hypergraph} w/ GroundTruth & 0.221 \scriptsize{$\pm$ 0.000} & 0.359 \scriptsize{$\pm$ 0.000} \\
    \bottomrule
\end{tabular}
\end{subtable}
\hspace{-14pt}
\begin{subtable}{.43\linewidth}
\centering
% \captionsetup{font=large,labelfont=large}
\captionsetup{font=large}
\caption{Product Return Prediction (F1$\uparrow$)}
\begin{tabular}{l|c}
    \toprule
    Method & Etail \\
    \midrule
    HyperGO \cite{li2018tail} w/o Labels & 0.718 \scriptsize{$\pm$ 0.000} \\
    \midrule
    HyperGO \cite{li2018tail} w/ WHATsNet & 0.723 \scriptsize{$\pm$ 0.003} \\
    HyperGO \cite{li2018tail} w/ CoNHD & \cellcolor{lightgray} \textbf{0.733 \scriptsize{$\pm$ 0.004}} \\
    \midrule
    HyperGO \cite{li2018tail} w/ GroundTruth & 0.738 \scriptsize{$\pm$ 0.000} \\
    \bottomrule
\end{tabular}
\end{subtable}
\end{adjustbox}
% \vspace{-4mm}
\end{table*}

To examine whether CoNHD learns separable embeddings for the same node across different edges, we follow \citep{choe2023classification} and use LDA to visualize embeddings of the largest-degree node in Email-Enron. Fig.~\ref{fig:visualization_embeddings} shows that CoNHD learns more separable embeddings than WHATsNet. Compared to WHATsNet, CoNHD provides adaptive representation sizes by introducing co-representations, which can avoid information loss for large-degree nodes.

\textbf{Efficiency.} The performance and training time on Email-Enron and Email-Eu are illustrated in Fig.~\ref{fig:training_time}, tested on a single NVIDIA A100 GPU. As full-batch training is impractical for large hypergraphs, we only compare models using mini-batch training. The best baseline WHATsNet trades efficiency for performance, while CoNHD not only achieves the best performance but also maintains high efficiency. Different from message passing, in CoNHD, the within-edge and within-node interactions are designed as parallel without inter-dependency, which is helpful for restricting inputs from only direct neighbors to the target node-edge pair and improving efficiency. Additionally, CoNHD eliminates the edge-dependent feature extraction and aggregation steps in edge-dependent message passing methods, further increasing efficiency.

\begin{figure}[t]
\centering
\includegraphics[width=1\linewidth]{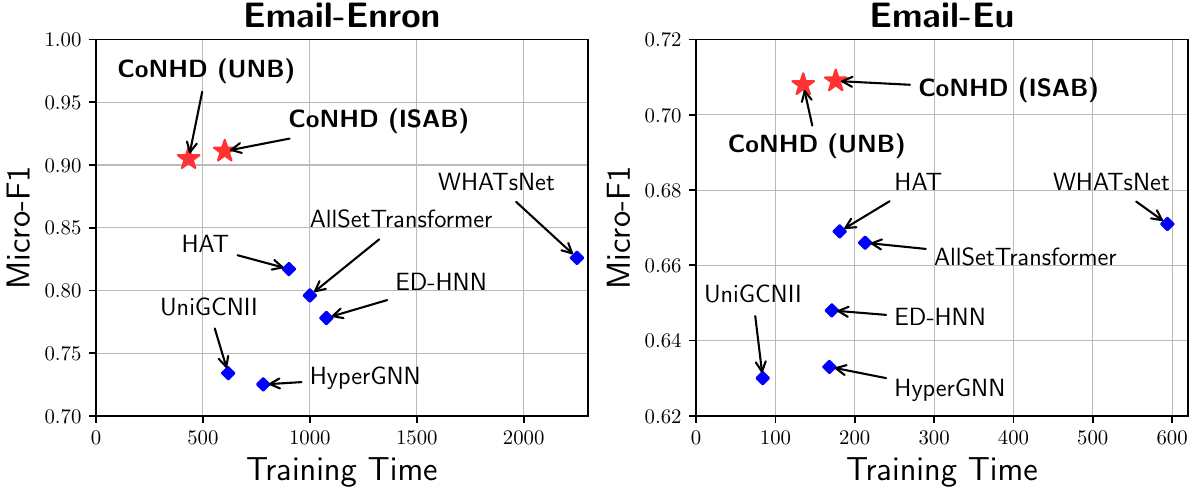}
\caption{\textbf{Comparison of the performance and training time (minutes).} CoNHD demonstrates significant improvements in terms of Micro-F1 while maintaining good efficiency.}
\label{fig:training_time}
% \vspace{-14pt}
\end{figure}

\subsection{Performance of Constructing Deep HGNNs}

Oversmoothing is a well-known issue in deep HGNNs \citep{wang2023equivariant, yan2024hypergraph}, which hinders the utilization of long-range information and limits performance. Diffusion-based HGNNs are shown to be more robust against this issue \citep{wang2023equivariant, chamberlain2021grand}. To evaluate the performance of CoNHD in constructing deep HGNNs, we conduct experiments on ENC with varying numbers of HGNN layers. As deeper HGNNs significantly increase GPU memory usage, we experiment on Citeseer-Outsider, a relatively small dataset, to ensure computational feasibility.

As shown in Fig.~\ref{fig:depth_experiment}, the performance of WHATsNet drops sharply beyond 4 layers, while two diffusion-based methods, EDHNN and HDS$^{ode}$, remain stable but show no notable gains with deeper architectures. In contrast, CoNHD continues to improve with more layers, and the performance converges after 16 layers. This suggests that CoNHD can effectively leverage long-range information to enhance performance. Compared to other diffusion-based HGNNs, the co-representation design in CoNHD allows the same node to have distinct representations when interacting within different hyperedges, ensuring the updated information to different node-edge pairs remains diverse and thereby preventing the collapse into oversmoothed representations.

\subsection{Application to Downstream Tasks}

The ENC task has been shown to be beneficial for many downstream applications \citep{choe2023classification}. Following \citep{choe2023classification}, we evaluate whether the ENC labels predicted by CoNHD can improve three downstream tasks: Ranking Aggregation (Halo, AMiner), Clustering (DBLP, AMiner), and Product Return Prediction (Etail). The ENC labels are first predicted and then used as additional input features to enhance the downstream prediction performance. Since the improvements depend not only on model performance in ENC but also on the relevance between downstream tasks and ENC, we report both the ENC prediction results and the downstream task performance. 

\begin{figure}[t]
\centering
\includegraphics[width=0.75\linewidth]{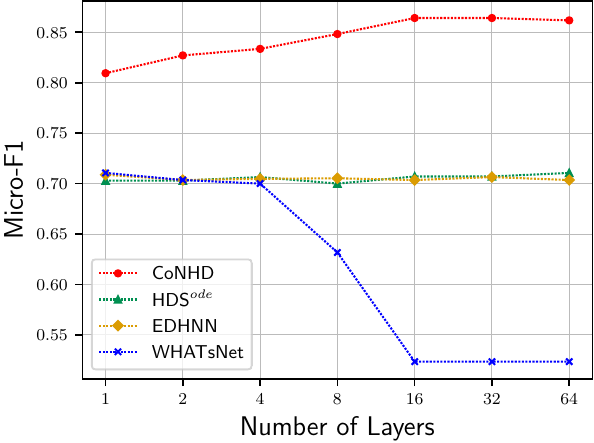}
\caption{\textbf{Performance of HGNNs with varying numbers of layers on the Citeseer-Outsider dataset.} CoNHD achieves the best performance across all settings.}
\label{fig:depth_experiment}
\end{figure}

Consistent with results in Section~\ref{sec:exp_ENC}, Table~\ref{tab:ENC_performance_downstream} shows that CoNHD consistently achieves superior performance on ENC label prediction across all datasets. Table~\ref{tab:performance_downstream} further demonstrates that using predicted ENC labels as additional information improves downstream task performance compared to cases where these labels are not used. CoNHD also outperforms WHATsNet across all three downstream tasks, benefiting from more accurate ENC label prediction. In the ranking aggregation task, using predicted labels even surpasses the cases using ground truth labels. This suggests that the ground truth labels might contain noise, while the predicted labels better capture the underlying smooth structure of the label space and further enhance the downstream task performance.

% \vspace{-2mm}
\subsection{Effectiveness of the Equivariant Design}\label{sec:exp_ablation_study}

CoNHD reformulates interactions as multi-output functions by introducing co-representations as shown in Fig.~\ref{fig:single_output_multi_output}, thereby enabling the use of equivariant functions. To show the importance of the equivariant design, we apply a mean aggregation to the equivariant outputs, reducing $\phi$ and $\varphi$ to invariant functions with identical outputs for different node-edge pairs. We conduct experiments on Email-Enron and Email-Eu.

\begin{table}
% \vspace{-13pt}
\small{\caption{\textbf{Effectiveness of the equivariance in two diffusion operators $\phi$ and $\varphi$.} We use \cmark~and~\xmark~to denote whether the corresponding operator is equivariant or invariant, respectively. \setlength{\fboxsep}{1pt}\colorbox{lightgray}{Shaded cells} indicate the variants with equivariance significantly outperform the one with only invariant operators.}\label{tab:ablation_study}}
% \vspace{-3pt}
\centering
\begin{adjustbox}{max width =1.0\linewidth}
\begin{tabular}{c|cc|cccc}
\toprule
\multirow{2}{*}{Method} & \multirow{2}{*}{$\phi$} & \multirow{2}{*}{$\varphi$} & \multicolumn{2}{c}{Email-Enron} & \multicolumn{2}{c}{Email-Eu} \\
% \cmidrule(lr){4-5} \cmidrule(lr){6-7}
& & & Micro-F1 & Macro-F1 & Micro-F1 & Macro-F1 \\
\midrule
\multirow{4}{*}{CoNHD (UNB)} & \xmark & \xmark & 0.827 \scriptsize{$\pm$ 0.000} & 0.769 \scriptsize{$\pm$ 0.004} & 0.673 \scriptsize{$\pm$ 0.000} & 0.645 \scriptsize{$\pm$ 0.001} \\
& \xmark & \cmark & \cellcolor{lightgray} 0.876 \scriptsize{$\pm$ 0.001} & \cellcolor{lightgray} 0.817 \scriptsize{$\pm$ 0.006} & \cellcolor{lightgray} 0.698 \scriptsize{$\pm$ 0.001} & \cellcolor{lightgray} 0.677 \scriptsize{$\pm$ 0.002} \\
& \cmark & \xmark & \cellcolor{lightgray} 0.903 \scriptsize{$\pm$ 0.001} & \cellcolor{lightgray} 0.855 \scriptsize{$\pm$ 0.004} & \cellcolor{lightgray} 0.707 \scriptsize{$\pm$ 0.000} & \cellcolor{lightgray} 0.688 \scriptsize{$\pm$ 0.002} \\
& \cmark & \cmark & \cellcolor{lightgray} \textbf{0.905 \scriptsize{$\pm$ 0.001}} & \cellcolor{lightgray} \textbf{0.858 \scriptsize{$\pm$ 0.004}} & \cellcolor{lightgray} \textbf{0.708 \scriptsize{$\pm$ 0.001}} & \cellcolor{lightgray} \textbf{0.689 \scriptsize{$\pm$ 0.001}} \\
\midrule
\multirow{4}{*}{CoNHD (ISAB)} & \xmark & \xmark & 0.829 \scriptsize{$\pm$ 0.001} & 0.765 \scriptsize{$\pm$ 0.007} & 0.673 \scriptsize{$\pm$ 0.001} & 0.647 \scriptsize{$\pm$ 0.002} \\
& \xmark & \cmark & \cellcolor{lightgray} 0.878 \scriptsize{$\pm$ 0.001} & \cellcolor{lightgray} 0.823 \scriptsize{$\pm$ 0.005} & \cellcolor{lightgray} 0.698 \scriptsize{$\pm$ 0.001} & \cellcolor{lightgray} 0.678 \scriptsize{$\pm$ 0.003} \\
& \cmark & \xmark & \cellcolor{lightgray} 0.910 \scriptsize{$\pm$ 0.001} & \cellcolor{lightgray} 0.870 \scriptsize{$\pm$ 0.003} & \cellcolor{lightgray} 0.707 \scriptsize{$\pm$ 0.001} & \cellcolor{lightgray} 0.689 \scriptsize{$\pm$ 0.001} \\
& \cmark & \cmark & \cellcolor{lightgray} \textbf{0.911 \scriptsize{$\pm$ 0.001}}  & \cellcolor{lightgray} \textbf{0.871 \scriptsize{$\pm$ 0.002}} & \cellcolor{lightgray} \textbf{0.709 \scriptsize{$\pm$ 0.001}} & \cellcolor{lightgray} \textbf{0.690 \scriptsize{$\pm$ 0.002}} \\
\bottomrule
\end{tabular}
\end{adjustbox}
% \vspace{-10pt}
\end{table}

As shown in Table~\ref{tab:ablation_study}, CoNHD with two equivariant operators achieves the best performance, significantly outperforming the variant with two invariant operators. Furthermore, variants with just one equivariant operator still outperform the fully invariant model, indicating that equivariance benefits both within-edge and within-node interactions. We also notice that the performance gap between the full equivariant model and the variant with only the equivariant within-edge operator $\phi$ is not significant. This might imply that within-edge interactions can provide the majority of the information needed for predicting the ENC labels in these datasets. 

%%%%%%%%%%%%%%%%%%%%%%%%%%%%%%%%%%%%%%%%%%%%%%%%%%%%%%%%%%%%%%%%%%%%%%%%%%%%%%%
% Conclusion
%%%%%%%%%%%%%%%%%%%%%%%%%%%%%%%%%%%%%%%%%%%%%%%%%%%%%%%%%%%%%%%%%%%%%%%%%%%%%%%
\section{Conclusion}\label{sec:conclusion}

In this paper, we develop CoNHD, a novel diffusion-based HGNN for modeling edge-specific features in ENC. CoNHD reformulates within-edge and within-node interactions as multi-output equivariant diffusion processes among node-edge co-representations, which disentangles edge-specific features and provides adaptive representation sizes. Our experiments demonstrate that CoNHD achieves the best performance on ten benchmark ENC datasets and several downstream tasks without sacrificing efficiency. We further show the robustness of CoNHD against the oversmoothing issue and validate the effectiveness of the equivariant design. Future work could explore extending CoNHD to more complex scenarios, such as dynamic hypergraphs \citep{yin2022dynamic} and multi-modal hypergraphs~\citep{kim2020hypergraph}, where existing approaches mostly rely on traditional message passing-based HGNNs \citep{sun2022structure, cheng2024retrieval}. CoNHD has the potential to improve representation quality by modeling edge-specific features in these complex settings.

\begin{acks}
This work is supported by the AI4Intelligence project with file number KICH1.VE01.20.011, partly financed by the Dutch Research Council (Nederlandse Organisatie voor Wetenschappelijk Onderzoek, NWO).
\end{acks}

%%%%%%%%%%%%%%%%%%%%%%%%%%%%%%%%%%%%%%%%%%%%%%%%%%%%%%%%%%%%%%%%%%%%%%%%%%%%%%%
% GenAI Usage Disclosure
%%%%%%%%%%%%%%%%%%%%%%%%%%%%%%%%%%%%%%%%%%%%%%%%%%%%%%%%%%%%%%%%%%%%%%%%%%%%%%%
\section*{GenAI Usage Disclosure}\label{sec:genai_usage_disclosure}

During the writing process, we employed ChatGPT solely for language refinement and grammatical corrections. All technical concepts, experimental work, and analytical content were independently conducted and written by the authors without relying on generative AI for idea generation or content creation.

\bibliographystyle{ACM-Reference-Format}
\balance
\bibliography{main}

\end{document}